\documentclass{article} 
\usepackage{iclr2024_conference,times}

\usepackage{amsmath,amsfonts,bm}

\def\eqref#1{equation~\ref{#1}}

\def\1{\bm{1}}

\DeclareMathAlphabet{\mathsfit}{\encodingdefault}{\sfdefault}{m}{sl}
\SetMathAlphabet{\mathsfit}{bold}{\encodingdefault}{\sfdefault}{bx}{n}

\newcommand{\E}{\mathbb{E}}

\newcommand{\R}{\mathbb{R}}

\usepackage{hyperref}
\usepackage{url}

\usepackage{enumitem}
\usepackage{wrapfig}
\usepackage{algorithm}
\usepackage{algpseudocode}
\usepackage{amssymb, amsmath, amsthm}
\usepackage{graphics}
\usepackage{epsfig}
\usepackage{graphicx}
\usepackage{subfigure}
\usepackage{booktabs}
\usepackage{float}
\usepackage{color}
\usepackage{enumitem}
\usepackage{xurl}
\usepackage{multirow}
\newcommand{\real}{\mathbb{R}}
\renewcommand{\paragraph}[1]{\textbf{#1.}}
\usepackage{stfloats}

\newtheorem{theorem}{Theorem}[section]

\newcommand{\ourmodel}{StructComp}

\iclrfinalcopy
\title{StructComp: Substituting propagation with Structural Compression in Training Graph Contrastive Learning}

\author{Shengzhong Zhang \\
Fudan University, Shanghai, China\\
\texttt{szzhang17@fudan.edu.cn} \\
\And
Wenjie Yang \\
Fudan University, Shanghai, China\\
\texttt{yangwj22@m.fudan.edu.cn} \\
\AND
 Xinyuan Cao \\
Georgia Institute of Technology, Midtown, USA \\
\texttt{xcao78@gatech.edu} \\
\And
 Hongwei Zhang\\
Fudan University, Shanghai, China\\
\texttt{hwzhang22@m.fudan.edu.cn} \\
\AND
Zengfeng Huang\thanks{Corresponding author} \\
Fudan University, Shanghai, China\\
\texttt{huangzf@fudan.edu.cn} \\
}

\begin{document}

\maketitle

\begin{abstract}

Graph contrastive learning (GCL) has become a powerful tool for learning graph data, but its scalability remains a significant challenge. In this work, we propose a simple  yet  effective training framework called Structural Compression (StructComp) to address this issue. Inspired by a sparse low-rank approximation on the diffusion matrix, \ourmodel~trains the encoder with the compressed nodes. This allows the encoder not to perform any message passing during the training stage, and significantly reduces the number of sample pairs in the contrastive loss. We theoretically prove that the original GCL loss can be approximated with the contrastive loss computed by \ourmodel. Moreover, \ourmodel~can be regarded as an additional regularization term for GCL models, resulting in a more robust encoder. Empirical studies on various datasets show that StructComp greatly reduces the time and memory consumption while improving model performance compared to the vanilla GCL models and scalable training methods. 
    
\end{abstract}

\section{introduction}\label{sec:intro}
Graph neural networks (GNNs) \citep{kipf2016semi,velickovic2017graph,chen2020simple,liu2020towards} provide powerful tools for analyzing complex graph datasets and are widely applied in various fields such as recommendation system \citep{cai2023lightgcl}, social network analysis \citep{zhang2021cvp}, traffic flow prediction \citep{Wang2020traffic}, and molecular property prediction \citep{alex2017protein}. However, the scalability limitation of GNNs hampers their extensive adoption in both industrial and academic domains. This challenge is particularly pronounced within the realm of unsupervised graph contrastive learning (GCL). Compared to the research on the scalability of supervised GNNs \citep{hamilton2017inductive, chen2018fastgcn,chen2018stochastic,zou2019layer,cong2020minimal,ramezani2020gcn,markowitz2021graph, zeng2020graphsaint,chiang2019cluster,zeng2021decoupling}, there is little attention paid to the scalability of GCL \citep{wang2022adagcl, zheng2022ggd}, and there is no universal framework for training various GCL models.

The scalability issue of GCL mainly has two aspects: Firstly, the number of nodes that need to be computed in message passing grows exponentially. Secondly, GCL usually requires computation of a large number of sample pairs, which may require computation and memory quadratic in the number of nodes.  At the same time, the graph sampling \citep{zeng2020graphsaint,chiang2019cluster,zeng2021decoupling} and decoupling technology \citep{wu2019simplifying,zhu2021ssgc} used for supervised GNN training are not applicable to GCL. Graph sampling might affect the quality of positive and negative samples, thereby reducing the performance of the model. As GCL usually involves data augmentation, the decoupling method that precomputes the diffusion matrix is not feasible.

In order to solve the two aforementioned problems simultaneously, we use a sparse assignment matrix to replace message passing, which is a low-rank approximation of the diffusion matrix. Utilizing the assignment matrix, we can compute the mixed node features, which contain all the local node information and can be regarded as the community center feature. By controlling the similarity of the community center embeddings, we can make node embeddings in similar communities close to each other, and node embeddings in dissimilar communities distant from each other. Since the number of sample pairs needed for computation by mixed nodes is significantly less than that for full nodes, and the encoder no longer computes message passing, the computational resources required for training are greatly saved.

Specifically, we propose an extremely simple yet effective GCL training framework, called Structural Compression (\ourmodel). This framework is applicable to various single-view GCL models and multi-view GCL models. During the training process, the GNN encoder does not need to perform message passing, at which point the encoder can be regarded as a MLP. Our model takes the compressed features as input and trains in the same way as the corresponding GCL model (i.e., the same loss function, optimization algorithm). In the inference process, we take the complete graph structure information and node features as input, and use the GNN encoder to obtain the embedding representations of all nodes.

Our contributions are summarized as follows:
\begin{enumerate}
\item We propose a novel GCL training framework, \ourmodel. Motivated by a low-rank approximation of the adjacency matrix, \ourmodel~significantly improves the scalability of GCLs by substituting message-passing with node compression. \ourmodel~trains MLP encoder on these mixed nodes and later transfers parameters to GNN encoder for inference. 

\item We customize a data augmentation method specifically for \ourmodel, making \ourmodel~adaptive to both single-view and multi-view GCL models. To the best of our knowledge, \ourmodel~is the first unified framework designed specifically for GCL training.

\item We theoretically guarantee that the compressed contrastive loss can be used to approximate the original graph contrastive loss. And we prove that our method introducing an extra regularization term into the scalable training, which makes the model more robust.
 
\item We empirically compare \ourmodel~with full graph training and other scalable training methods under four GCL models. Experimental results on seven datasets demonstrate that \ourmodel~improves the GCL model's performance, and significantly reduces memory consumption and training time.
\end{enumerate}

\section{Preliminaries}
\noindent\textbf{Notation.} Consider an undirected graph $G=(A, X)$, where $A\in\{0,1\}^{n\times n}$ represents the adjacency matrix of $G$, and $X\in \real^{n\times d}$ is the feature matrix. The set of vertices and edges is represented as $V$ and $E$, with the number of vertices and edges given by $n=|V|$ and $m=|E|$, respectively. The degree of node $v_i$ denoted as $d_i$. The degree matrix $D$ is a diagonal matrix and its $i$-th diagonal entry is $d_i$. 

\noindent\textbf{Graph neural network encoders}. The GNN encoders compute node representations by aggregating and transforming information from neighboring nodes. One of the most common encoders is the Graph Convolutional Network (GCN) \citep{kipf2016semi}, and its propagation rule is defined as follows:
\begin{equation}
	H^{(l+1)}=\sigma \left( \widetilde{D}^{-\frac{1}{2}}\widetilde{A}\widetilde{D}^{-\frac{1}{2}}H^{(l)}W^{(l)} \right),
\end{equation}
where $\widetilde{A} = A+I$, $\widetilde{D}=D+I$ and $W^{(l)}$ is a learnable parameter matrix. GCNs consist of multiple convolution layers of the above form, with each layer followed by an activation $\sigma$ such as ReLU. 

\noindent\textbf{Graph contrastive learning}.
Graph contrastive learning is an unsupervised graph representation learning method. Its objective is to learn the embeddings of the graph by distinguishing between similar and dissimilar nodes. Common methods of graph contrastive learning can be divided into two types: single-view graph contrastive learning \citep{zhang2020sce, zhu2021coles} and multi-view graph contrastive learning \citep{zhu2020grace, zhu2021gca, zhang2021ccassg, zheng2022ggd}.

In single-view graph contrastive learning, positive and negative sample pairs are generated under the same view of a graph. In this case, the positive samples are typically pairs of adjacent or connected nodes, while the negative samples are randomly selected pairs of non-adjacent nodes. Then, through a GNN encoder and a contrastive loss function, the algorithm learns to bring the embedding vectors of positive sample pairs closer and push the embedding of negative sample pairs further apart. The common single-view contrastive loss function \citep{hamilton2017inductive} of node $u$ is as follows:
\begin{equation}
\begin{aligned}
\mathcal{L}(u) = -\log(\sigma(z^T_u z_v ))-{\sum_{k=1}^{K} \log(\sigma(-z^T_u z_k ))}.
\end{aligned}
\end{equation}
Here, node $v$ is the positive sample of node $u$, node $k$ is the negative sample of node $u$, and $K$ represents the number of negative samples.

Multi-view graph contrastive learning uses different views of the graph to generate sample pairs. These views can be contracted by different transformations of the graph, such as DropEdge \citep{rong2019dropedge} and feature masking \citep{zhu2020grace}. We generate the embeddings of the nodes, aiming to bring the embedding vectors of the same node but from different views closer, while pushing the embedding vectors from different nodes further apart. The common multi-view contrastive loss function of each positive pair $(u,v)$ is as follows:
\begin{equation}
\begin{aligned}
\mathcal{L}(u, v) = 
 \log \frac {e^{\phi\left(z_u, z_v \right) / \tau}} {e^{\phi\left(z_u, z_v \right) / \tau} + \sum_{k\neq u, k \in G_1} e^{\phi\left(z_u, z_k \right) / \tau} + \sum_{k \neq u, k \in G_2} e^{\phi\left(z_u,  z_k \right) / \tau}}.
\end{aligned}\label{eqn:mv_loss}
\end{equation}

Here $u$ and $v$ represent the same node from different views, $\phi$ is a function that computes the similarity between two embedding vectors. $G_1$ and $G_2$ are two different views of the same graph. $\tau$ is temperature parameter.

\section{Structural Compression}
\subsection{motivation}\label{subsec:motivation}

To reduce the training complexity of GCLs, we start with a low-rank approximation $C$ of the adjacency matrix $\hat{A}^k$, such that $\hat{A}^k=CC^T$. Although the complexity of matrix multiplication is significantly reduced by the approximation, the actual training time will not decrease due to the dense nature of $C$. Moreover, the amount of negative pairs needed for the contrastive learning remains $O(n^2)$. To address the above issues simultaneously, we introduce a sparse constraint for the low-rank approximation and force $C$ to be a graph partition matrix $P' \in \mathbb{R}^{n\times n'}$ ($P'_{ij}=1$ if and only if the node $i$ belongs to cluster $j$), where $n'$ is the number of clusters in the partition. Using $P' P^TX$ to approximate $\hat{A}^k X$ (where $P$ is the row-normalized version of $P'$), a key advantage is that nodes in the same cluster share the same embedding, and $P^TX$ contains all the information needed to compute the loss function. Therefore, we only compute $P^TX\in \mathbb{R}^{n'\times d}$: a ``node compression" operation, where nodes in the same cluster are merged together. Since nodes in the same cluster share their embeddings, performing contrastive learning on these compressed nodes is equivalent to that on the nodes after the low-rank propagation (i.e., $P' P^TX$). Thus, the number of negative pairs reduced to $n'^2$. Moreover, the complexity of matrix multiplication is now down to $O(n)$ while persevering the sparsity. In a nutshell, two major challenges for scalable GCL in Section \ref{sec:intro} can be solved simultaneously by our method.

To generate the graph partition matrix $P$, we need to solve the following optimization problem:
\begin{equation}
\label{eq:opt}
\begin{split}
&\mathrm{minimize} \quad \Vert P' P^T - \hat{A}^k \Vert,\\
&\mathrm{subject\ to} \quad P' \in \{0,1\}^{n \times n'},P'1_{n'}=1_n.\\
\end{split}
\end{equation}
Intuitively, $P' P^T$ is a normalized graph that connects every pair of nodes within a community while discarding all inter-community edges. Thus, $\Vert P' P^T - \hat{A} \Vert$ equals the number of inter-community edges plus the number of disconnected node pairs within communities. Minimizing the former is the classic minimum cut problem, and minimizing the latter matches well with balanced separation, given the number of nodes pair grows quadratically. These objective aligns well with off-the-shelf graph partition methods, so we directly utilize METIS to produce $P$. 

From the perspective of spatial domain, we always hope to obtain an embedding $f(\hat{A}^kXW)$ of the following form: embeddings of nodes of the same class are close enough, while embeddings of nodes of different classes are far apart. Intuitively, this goal can be simplified to the class centers of the various class node embeddings being far apart, and the similarity of nodes within the same class being as high as possible. In other words, we can use the community center $f(P^TXW)$ as the embedding that needs to be computed in the loss function, instead of using all nodes in the community for computation. On the other hand, if the embeddings of nodes within the community are identical, we no longer need to compute the similarity of nodes within the community. Considering these, it is natural to use $P^TX$ in place of $\hat{A}^kX$ to compute the loss function. Moreover, since $P$ is solved based on graph partitioning, the result of the graph partitioning can facilitate the construction of positive and negative samples. The idea of using structural compression as a substitute of message-passing can be extended to a multi-layer and non-linear GNN, which is shown in Appendix \ref{appendix:nonlinear}.

\subsection{Framework of \ourmodel}

\textbf{Preprocessing}. We carry out an operation termed ``node compression". Based on the above analysis, we use the METIS algorithm to obtain the graph partition matrix, and then take the mean value of the features of the nodes in each cluster as the compressed feature, i.e., $X_c=P^T X$. After computing the compressed features, we also construct a compressed graph, i.e., $A_c=P^T AP$. Each node in $A_c$ represents a cluster in the original graph, and the edges represent the relationships between these clusters. $A_c$ is only used when constructing the contrastive loss and is \textbf{not} involved in the computations related to the encoder. To make our \ourmodel~adaptive to different types of GCL models, we carefully design some specific modules for single-view and multi-view GCLs, respectively.

\begin{figure*}[h]
\setlength{\abovecaptionskip}{-0.3cm}
	\centering
{\includegraphics[width=1.0\textwidth]{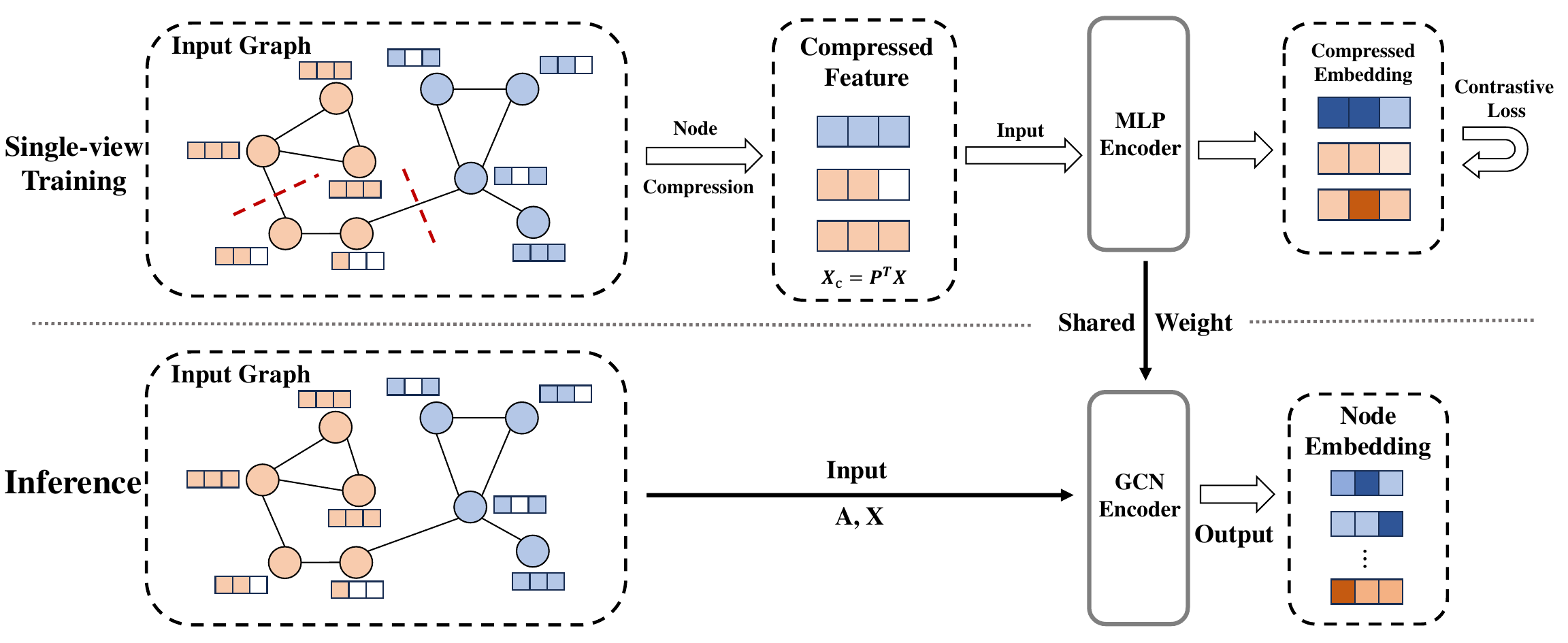}}
\caption{The overall framework of single-view \ourmodel.}
 	\label{sv_inf}
\end{figure*}

\textbf{Single-view \ourmodel.} In single-view graph contrastive learning, we use the preprocessed compression features $X_c$ as input, and replace the GNN encoders with MLP encoders. For instance, a two-layer neural network and embedding can be represented as follows:
\begin{equation}
Z_c=\sigma(\sigma (X_c W_1)W_2)
\end{equation}

We proposed to sample positive and negative pairs based on the compressed graph $A_c$ instead of $A$. One additional advantages of using the compressed graph over the original graph is that it significantly improves the accuracy of negative pairs sampling. For instance, since highly connected nodes are compressed together, they are not able to be selected as negative pairs. Then we use the same loss function and optimization algorithm as original GCL models to optimize the single-view contrastive learning loss $L(Z_c)$. Figure \ref{sv_inf} shows the flow chart of single-view \ourmodel.

Once the model is adequately trained, we transition to the inference phase. We revert the changes made during the training phase by replacing the MLP encoder back to the GNN encoder. Then, we input the complete graph structure information and node features to generate the embeddings for all nodes, as detailed below:
\begin{equation}
Z=\sigma(\hat{A}\sigma (\hat{A}X W_1)W_2).
\end{equation}
\textbf{Multi-view \ourmodel.} In multi-view contrastive learning, we need to compare two perturbed views of the graph. This requires us not only to compress the node features but also to apply data augmentation to these compressed features. 
However, traditional data augmentation methods such as DropEdge are not applicable in \ourmodel, as there is no longer an $A$ available for them to perturb during training. To fill this gap, we introduce a new data augmentation method called `DropMember'. This technique offers a novel way to generate different representations of compressed nodes, which are essential for multi-view contrastive learning under \ourmodel.

The DropMember method is implemented based on a predetermined assignment matrix P. For each node in the compressed graph, which represents a community, we randomly drop a portion of the nodes within the community and recalculate the community features. Formally, for each cluster $j$ in the augmented $X'_c$, we have:
\begin{equation}
x'_j=\frac{1}{s'}\sum_{i=1}^{s}m_ix_i.
\end{equation}
Here, $s$ represents the number of nodes contained in cluster $j$,  $m_i$ is independently drawn from a Bernoulli distribution and $s'=\sum_{i=1}^{s}m_i$. By performing contrastive learning on the compressed features obtained after DropMember and the complete compressed features, we can train a robust encoder. The loss of some node information within the community does not affect the embedding quality of the community.

\begin{figure*}[thbp]
\setlength{\abovecaptionskip}{-0.3cm}
	\centering
{\includegraphics[width=1.0\textwidth]{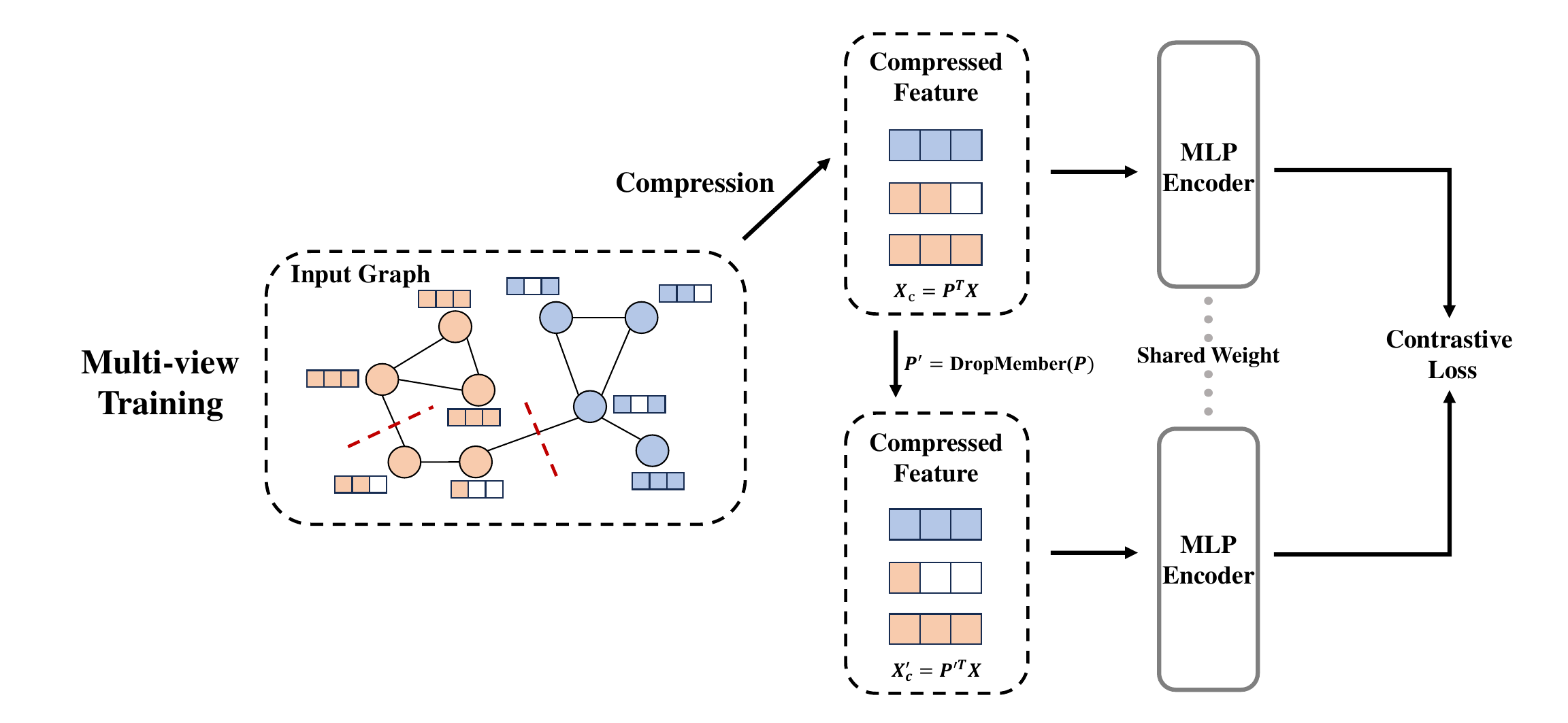}}
\caption{The training process of multi-view \ourmodel.}
 	\label{mv}
\end{figure*}

For the multi-view graph contrastive learning model, we need to compute the representations of two different views. Figure \ref{mv} shows the training process of multi-view \ourmodel. In our implementation, we use the complete $X_c$ and the perturbed $X'_c$ after DropMember as two different views. The embeddings of the two views are as follows:
\begin{equation}
\vspace{-1mm}
Z_c=\sigma(\sigma (X_c W_1)W_2), \quad Z'_c=\sigma(\sigma (X'_c W_1)W_2).
\vspace{-1mm}
\end{equation}

We use the same loss function and optimization algorithm as original multi-view GCL models to optimize the contrastive loss $\mathcal{L}(Z_c, Z'_c)$. Once the model is trained, the inference process of multi-view \ourmodel~is the same as that of single-view \ourmodel.

\section{Theory analysis of \ourmodel}

\subsection{The equivalence of the compressed loss and the original loss}

In this section, we demonstrate that the contrastive loss on the original graph is close to the sum of the compressed contrastive loss and the low-rank approximation gap. In other words, if the low-rank approximation in section \ref{subsec:motivation} is properly satisfied,  we can estimate the original GCL loss using the compressed contrastive loss.

Here we consider the Erd\H{o}s-Rényi model, denoted as $G(n,p)$, where edges between $n$ vertices are included independently with probability $p$. We use $\|\cdot \|_2$ to represent $l_2$ norm and $\|\cdot\|_F$ to represent Frobenius norm. Additionally, we denote the feature vector of each node as $X_i\in\R^d$. Then we can prove the following theorem, for simplicity, we only consider a one-layer message-passing and an unweighted node compression. We leave the proof and details of notations in Appendix \ref{appendix:equivalence}.

\begin{theorem}
\label{thm:partition}
    For the random graph $G(n,p)$ from Erd\H{o}s-Rényi model, we
    construct an
    even partition $\mathcal{P}=\{S_1,\cdots,S_{n'}\}$. Let  $f_G(X)=AXW$ be a feature mapping in the original graph and $f_\mathcal{P}(X)=P^{' T} XW$ as a linear mapping for the mixed nodes, where $W\in\R^{d\times d'}$. Then by conducting single-view contrastive learning, the contrastive loss for the original graph, denoted as $\mathcal{L}_G(W)$, can be approximated as the sum of the compressed contrastive loss, $\mathcal{L}_\mathcal{P}(W)$, and a term related to the low-rank approximation. Assume the features are bounded by $S_X := \max_i\|X_i\|_2$, we have
    \[
    \lvert\mathcal{L}_G(W) - \mathcal{L}_\mathcal{P}(W) \rvert 
    \leq \|A-P^{'}P^{' T}\|_F S_X\|W\|_2.
    \]
\end{theorem}

For a similar upper bound without the Erd\H{o}s-Rényi graph assumption, please refer to Appendix \ref{appendix:nonrandom}.
\subsection{The regularization introduced by \ourmodel} \label{sec:regulation_theory}

Following \cite{fang2023dropmessage}, we show that multi-view \ourmodel~is equivalent to random masking on the message matrices $M$, where $M_{i,j}=\psi(h_i, h_j, e_{i,j})$ and $\psi$ is a function that takes the edges $e_{i,j}$ and the attached node representations $h_i, h_j$. First, the low rank approximation $||P' P-\hat{A}^k||$ is dropping the inter-cluster edges $E_{\mathrm{drop}}=\{E_{i,j}|A^k_{i,j}=1~\mathrm{and}~S(i)\neq S(j)\}$, where $S(i)$ denote the cluster that node $i$ belongs to. And the latter is then equivalent to DropMessage $M_{\mathrm{drop}}=\{ M_i|\mathrm{edge}(M_i)\in E_{\mathrm{drop}}\}$, where $\mathrm{edge}(M_i)\in E_{\mathrm{drop}}$ indicates which edge that $M_i$ corresponds to. Our DropMember for the cluster $c$ is dropping $V^c_{\mathrm{drop}}=\{ X_i|\epsilon_i=0~\mathrm{and}~S(i)=c\}$. This is  equivalent to $M_{\mathrm{drop}}=\{ M_i|\mathrm{node}(M_i)\in \bigcup_c V^c_{\mathrm{drop}}\}$. Then we have the following theorem:

\begin{theorem}\label{the:regulation}
    Consider a no-augmentation InfoNCE loss,
    \begin{equation}
        \mathcal{L}_{\mathrm{InfoNCE}}=\sum_i \sum_{j\in \mathrm{pos}(i)}[h_i^T h_j] + \sum_i \sum_{j\in \mathrm{neg}(i)} [\log (e^{h_i^T h_i}+e^{h_i^T h_j})].
    \end{equation}
    Optimizing the expectation of this with augmentation $\mathbb{E}[\Tilde{\mathcal{L}}_{\mathrm{InfoNCE}}]$ introduce an additional regularization term, i.e.,
    \begin{equation}
         \mathbb{E}[\Tilde{\mathcal{L}}_{\mathrm{InfoNCE}}]=\mathcal{L}_{\mathrm{InfoNCE}}+\frac{1}{2}\sum_i \sum_{j\in \mathrm{neg}(i)}\phi(h_i,h_j)\mathrm{Var}(\Tilde{h}_i),
    \end{equation}
where $\phi(h_i,h_j)=\frac{(e^{h_i^2}h_i^2+e^{h_i h_j}h_j^2)(e^{h_i^2}+e^{h_i h_j})-(e^{h_i^2}h_i+e^{h_i h_j}h_j)^2}{2(e^{h_i^2}+e^{h_i h_j})^2}$.
\end{theorem}

Theorem \ref{the:regulation} shows that, multi-view \ourmodel~not only improves the scalability of GCLs training, but also introduces an additional regularization term into the InfoNCE loss. By optimizing the variance of the augmented representations, encoders trained with \ourmodel~are more robust to minor perturbation. Please refer to the Appendix \ref{appendix:regularization} for more details.

\section{related work}

\noindent\paragraph{Scalable training on graph} 
To overcome the scalability issue of training GNNs, most of the previous scalable GNN training methods use sampling techniques \citep{hamilton2017inductive, chen2018fastgcn,chen2018stochastic,zou2019layer,cong2020minimal,ramezani2020gcn,markowitz2021graph,zeng2020graphsaint,chiang2019cluster,zeng2021decoupling}, including node-wise sampling, layer-wise sampling, and graph sampling. The key idea is to train GNNs with small subgraphs instead of the whole graph at each epoch. However, graph sampling techniques are mainly used for training supervised GNNs and are not applicable to unsupervised GNNs, as it is difficult to guarantee the provision of high-quality positive and negative samples after sampling. Another direction for scalable GNNs is to simplify models by decoupling the graph diffusion process from the feature transformation. The diffusion matrix is precomputed, and then a standard mini-batch training can be applied \citep{bojchevski2020scaling, chen2020scalable, wu2019simplifying}. This preprocessing method is also not applicable to graph contrastive learning, as the adjacency matrix and feature matrix are perturbed in contrastive learning, which necessitates the repeated computation of the diffusion matrix, rather than only in preprocessing. Besides, methods represented by GraphZoom \citep{chen2018harp,liang2018mile,deng2019graphzoom} learn node embeddings on the coarsened graph, and then refine the learned coarsed embeddings to full node embeddings. These methods mainly consider graph structural information, only applicable to handling traditional graph embedding\citep{Perozzi:2014:DOL:2623330.2623732,DBLP:conf/kdd/GroverL16}, but are not suitable for GCL models. Most importantly, these methods require a lot of time to construct the coarsened graph, and the coarse-to-refine framework inevitably leads to information loss.

See the Appendix \ref{sec:more_related}  for a discussion of the related work on graph contrastive learning.

\section{Experiment}

\subsection{Experimental Setup}

The results are evaluated on night real-world datasets \citep{kipf2016semi, velivckovic2018deep, zhu2021gca, hu2020ogb}, Cora, Citeseer, Pubmed, Amazon Computers, Amazon Photo, Ogbn-Arixv, Ogbn-Products and Ogbn-Papers100M. On small-scale datasets, including Cora, Citeseer, Pubmed, Amazon Photo and Computers, performance is evaluated on random splits. We randomly select 20 labeled nodes per class for training, while the remaining nodes are used for testing. All results on small-scale datasets are averaged over 50 runs, and standard deviations are reported.  For Ogbn-Arixv, Ogbn-Products and  Ogbn-Papers100M, we use fixed data splits as in previous studies \cite{hu2020ogb}. More detailed statistics of the night datasets are summarized in the Appendix \ref{sec:details}.

We use \ourmodel~to train two representative single-view GCL models, SCE \citep{zhang2020sce} and COLES  \citep{zhu2021coles}, and two representative multi-view GCL models, GRACE \citep{zhu2020grace} and CCA-SSG \citep{zhang2021ccassg}. To demonstrate the effectiveness of \ourmodel, we compare the classification performance of the original models and \ourmodel~trained models on small-scale datasets. For scalability on large graphs, we compare \ourmodel~with three scalable training methods (i.e., Cluster-GCN \citep{chiang2019cluster}, Graphsaint \citep{zeng2020graphsaint} and Graphzoom \citep{deng2019graphzoom}). For all the models, the learned representations are evaluated by classifiers under the same settings.

The key hyperparameter of our framework is the number of clusters, which is set to [300, 300, 2000, 1300, 700, 20000, 25000, 5000] on night datasets, respectively. All algorithms and models are implemented using Python and PyTorch Geometric. More implementation details can be found in Appendix \ref{sec:details}. Additional discussions and experimental results are included in Appendix \ref{more_exp}.

\subsection{Experimental Results} 

\noindent\paragraph{Performance on small-scale datasets} Table \ref{acc_small} shows the model performance on small datasets using the full graph training and \ourmodel~training. The results show that \ourmodel~improves the performance of the model in the vast majority of cases, especially the multi-view GCL models. In the single-view GCL models, \ourmodel~improves the average accuracy of SCE and COLES by 0.4\% and 0.2\%, respectively. In the multi-view GCL models, \ourmodel~improves the average accuracy of GRACE and CCA-SSG by 2.6\% and 1.6\%, respectively. The observed performance improvement can be attributed to two main factors. First, \ourmodel~constructs high-quality positive and negative pairs, e.g., it ensures that highly-connected nodes are not erroneously selected as negative pairs in multi-view GCLs. Second, as mentioned in Section \ref{sec:regulation_theory}, \ourmodel~implicitly introduces regularization into the contrastive learning process, resulting in a more robust encoder.

\begin{table*}[!thbp]\small
\vspace{-3mm}
	\caption{Comparison between \ourmodel~and full graph training across four GCL models on small datasets. The performance is measured by classification accuracy. ``Ave $\Delta$'' is the average improvement achieved by \ourmodel.}\label{tab:node}
	\centering
\begin{tabular}{l|ccccc|c}\toprule
\textbf{Method}& \textbf{Cora} & \textbf{Citeseer}&\textbf{Pubmed}& \textbf{Computers}&\textbf{Photo}&\textbf{Ave $\Delta$}\\ 
\midrule
SCE&81.0$\pm$1.3 &71.7$\pm$1.1&76.5$\pm$2.8&79.2 $\pm$1.7&87.8 $\pm$1.4  \\
SCE$_{\text{\ourmodel}}$ &81.6$\pm$0.9 &71.5$\pm$1.0&77.2$\pm$2.9 &79.7 $\pm$1.7&88.2 $\pm$1.4&+0.4 \\
\midrule
COLES&81.7$\pm$0.9&71.2$\pm$1.2 & 74.6$\pm$3.4 &79.5$\pm$1.6&88.5$\pm$1.4   \\

COLES$_{\text{\ourmodel}}$ &81.8$\pm$0.8&71.6$\pm$0.9&75.3$\pm$3.1 &79.4$\pm$1.6 &88.5$\pm$1.4&+0.2 \\
\midrule
GRACE&78.5$\pm$0.9  & 68.9$\pm$1.0 & 76.1$\pm$2.8& 76.2$\pm$1.9& 85.1$\pm$1.6   \\
GRACE$_{\text{\ourmodel}}$&79.7$\pm$0.9 &70.5$\pm$1.0&77.2$\pm$1.4 &80.6$\pm$1.5 &90.0$\pm$1.1 &+2.6\\
\midrule
CCA-SSG& 79.2$\pm$1.4 & 71.8$\pm$1.0 &76.0$\pm$2.0&82.7$\pm$1.0 &88.7$\pm$1.1 \\
CCA-SSG$_{\text{\ourmodel}}$& 82.3$\pm$0.8&  71.6$\pm$0.9&78.3$\pm$2.5&83.1$\pm$1.4 &90.8$\pm$1.0&+1.6 \\
\bottomrule
\end{tabular}\label{acc_small}
\end{table*}

\noindent\paragraph{Time and memory usage for small-scale datasets} Table \ref{tab:time_and_mem} shows the improvements in runtime and memory usage of each GCL model.  \ourmodel~saves the memory usage of the GCL models on the Cora, Citeseer, Pubmed, Computers, and Photo datasets by 5.9 times, 4.8 times, 33.5 times, 22.2 times and 57.1 times, respectively. At the same time, the training is speeded up by 1.8 times, 2.1 times, 17.1 times, 10.9 times and 21.1 times, respectively. This improvement is particularly evident when the dataset is large. 
The memory consumption for GRACE on the Computers dataset is even reduced by two orders of magnitude. These results strongly suggest that our method significantly reduces time consumption and memory usage while enhancing model performance.

\begin{table*}[!thbp]\scriptsize
\vspace{-3mm}
	\caption{Time (s/epoch) and memory usage (MB) for GCL training on small-scale datasets. ``Ave  improvement'' is the proportion of training resources used by \ourmodel~to the resources used by full graph training.}\label{tab:time_and_mem}
	\centering
\begin{tabular}{lcc|cc|cc|cc|cc}\toprule
\multicolumn{1}{c}{\multirow{2}*{\textbf{Method}}}& \multicolumn{2}{c}{\textbf{Cora}} & \multicolumn{2}{c}{\textbf{Citeseer}}& \multicolumn{2}{c}{\textbf{Pubmed}}& \multicolumn{2}{c}{\textbf{Photo}}&\multicolumn{2}{c}{\textbf{Computers}}\\ 
\cmidrule(r){2-11}
  &Mem &Time& Mem&Time&Mem &Time&Mem &Time&Mem &Time\\ 
\midrule
SCE& 82 &0.003& 159&0.004& 1831&0.027 &329&0.006&920&0.015\\
SCE$_{\text{\ourmodel}}$ &23  &0.002 & 59&0.002&54 &0.003 & 16&0.002&29&0.002  \\
\midrule
COLES& 115&0.004&204&0.004&1851&0.033&378&0.015&1018&0.031 \\
COLES$_{\text{\ourmodel}}$&24&0.002&60&0.003&61&0.003&21 &0.003 &39&0.003 \\
\midrule
GRACE&441 &0.017 &714&0.025&11677&0.252&1996&0.106 &5943&0.197 \\
GRACE$_{\text{\ourmodel}}$&37&0.009&72&0.009&194&0.009&59&0.008&54&0.008\\
\midrule
CCA-SSG& 132&0.010 &225&0.011&825&0.123&1197&0.112&2418&0.210  \\
CCA-SSG$_{\text{\ourmodel}}$&38&0.006&71&0.005&85&0.006&41&0.005&40&0.005\\
\midrule 
Ave  improvement & 5.9$\times$&1.8$\times$&4.8$\times$&2.1$\times$&33.5$\times$&17.1$\times$&22.2$\times$&10.9$\times$&57.1$\times$&21.1$\times$\\
\bottomrule
\end{tabular}
\end{table*}

\noindent\paragraph{Scalability on large graphs} Table \ref{tab:arxiv} and Table \ref{tab:products} show the results of \ourmodel~on two large datasets Arxiv and Products, and compare them with three scalable training methods: Cluster-GCN, Graphsaint, and Graphzoom. We tested these training methods on four GCL models. Our method achieves the best performance on all GCL models on both datasets. The experimental results for Papers100M are listed in Appendix \ref{papers100m}. The models trained by our method achieve the highest accuracy, while also require significantly lower memory and time consumption compared to other models. These results highlight the effectiveness of our method in handling large-scale graph data.

\begin{table*}[thbp]\scriptsize
	\caption{Performance and training consumption (time in s/epoch and memory usage in GB) on the Ogbn-Arxiv dataset. Each model is trained by \ourmodel~and three scalable training frameworks.}\label{tab:arxiv}
	\centering
\begin{tabular}{lccc|ccc|ccc|ccc}\toprule
\multicolumn{1}{c}{\multirow{2}*{\textbf{Method}}}& \multicolumn{3}{c}{\textbf{SCE}} & \multicolumn{3}{c}{\textbf{COLES}}& \multicolumn{3}{c}{\textbf{GRACE}}& \multicolumn{3}{c}{\textbf{CCA-SSG}}\\ 
\cmidrule(r){2-13}
  &Acc &Time &Mem &Acc &Time &Mem&Acc &Time &Mem&Acc &Time &Mem  \\ 
\midrule
Cluser-GCN&70.4$\pm$0.2&10.8&4.2&71.4$\pm$0.2&13.0&4.2&70.1$\pm$0.1 & 3.5&17.3 &72.2$\pm$0.1&2.2&1.3\\
Graphsaint&70.4$\pm$0.2&7.6&9.0&71.3$\pm$0.1&26.9&22.5&70.0$\pm$0.2&3.6&14.1&72.1$\pm$0.1&4.4&1.8\\
Graphzoom& 70.6$\pm$0.2&0.04&9.8&70.0$\pm$0.3&0.08&14.3&68.2$\pm$0.3& 3.9& 13.7 & 71.3$\pm$0.2&1.0&3.4\\
\ourmodel&\textbf{71.6}$\pm$0.2&\textbf{0.03}&\textbf{1.8}&\textbf{71.8}$\pm$0.2 & \textbf{0.05}&\textbf{3.4} &\textbf{71.7}$\pm$0.2 & \textbf{1.2}&\textbf{11.7} &\textbf{72.2}$\pm$0.1 &\textbf{0.9}&\textbf{0.5}\\
\bottomrule
\end{tabular}
\end{table*}

\begin{table*}[thbp]\scriptsize
	\caption{Performance and training consumption (time in s/epoch and memory usage in GB) on the Ogbn-Products dataset. Each model is trained by \ourmodel~and three scalable training frameworks.}\label{tab:products}
	\centering
\begin{tabular}{lccc|ccc|ccc|ccc}\toprule
\multicolumn{1}{c}{\multirow{2}*{\textbf{Method}}}& \multicolumn{3}{c}{\textbf{SCE}} & \multicolumn{3}{c}{\textbf{COLES}}& \multicolumn{3}{c}{\textbf{GRACE}}& \multicolumn{3}{c}{\textbf{CCA-SSG}}\\ 
\cmidrule(r){2-13}
  &Acc &Time &Mem &Acc &Time &Mem&Acc &Time &Mem&Acc &Time &Mem \\ 
\midrule
Cluser-GCN&74.6$\pm$0.2&196.8&8.0&75.2$\pm$0.2&239.1&8.0&75.2$\pm$0.1&35.1&16.2&74.1$\pm$0.3&37.1&5.7\\
Graphsaint&74.5$\pm$0.3&18.5&9.5&75.3$\pm$0.1&20.6&9.9&75.4$\pm$0.2&36.8&14.3&74.8$\pm$0.4&35.4&3.3\\
Graphzoom&60.6$\pm$0.5&0.06&8.8&68.1$\pm$0.4&0.1&13.2&61.0$\pm$0.4&11.1&13.1&68.6$\pm$0.3&4.5&10.0\\
\ourmodel&\textbf{75.2}$\pm$0.1 &\textbf{0.05 }&\textbf{2.7} & \textbf{75.5}$\pm$0.1&\textbf{0.08} &\textbf{5.3} &\textbf{75.7}$\pm$0.1&\textbf{4.0}&\textbf{12.0}&\textbf{75.8}$\pm$0.2&\textbf{3.7}&\textbf{0.6}\\
\bottomrule
\end{tabular}
\end{table*}

\noindent\paragraph{Comparison of loss trends} To examine the equivalence of \ourmodel~training and traditional GCL training, we plug the parameters $U$ trained with the compressed loss $\mathcal{L}(X_c ;U)$ back into the GNN encoder and compute the complete loss function $\mathcal{L}(A,X;U)$. Figure \ref{fig:loss_trend} shows the trends of $\mathcal{L}(A,X;U)$ and the original loss $\mathcal{L}(A,X;W)$ trained with full graph. The behavior of the two losses matches astonishingly. This observation implies that, \ourmodel~produces similar parameters to the traditional full graph GCL training.

\noindent\paragraph{Effect of compression rate} We study the influence of the compression rate on the performance of \ourmodel. We use \ourmodel~to train four GCL models under different compression rates on Cora, Citeseer, and Pubmed. Figure \ref{fig:ratio} shows the performance change with respect to compression rates. We observed that when the compression rate is around 10\%, the performance of the models is optimal. If the compression rate is too low, it may leads to less pronounced coarsened features, thereby reducing the effectiveness of the training. 

\begin{figure*}[t]
\vspace{-3mm}
	\setlength{\abovecaptionskip}{-0.1cm}
	\setlength{\belowcaptionskip}{-0.5cm}
	\centering
  \subfigure[SCE on Cora]{\includegraphics[width=0.24\textwidth]{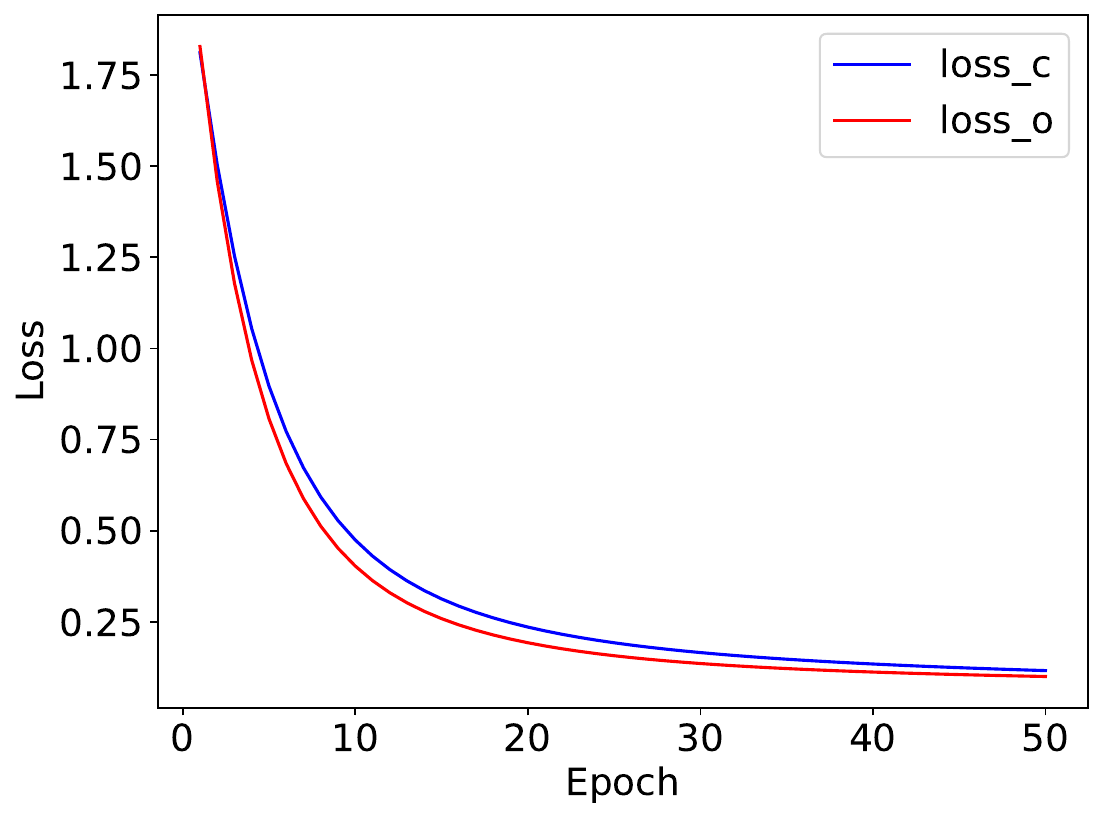}}
    \subfigure[COLES on Cora]{\includegraphics[width=0.24\textwidth]{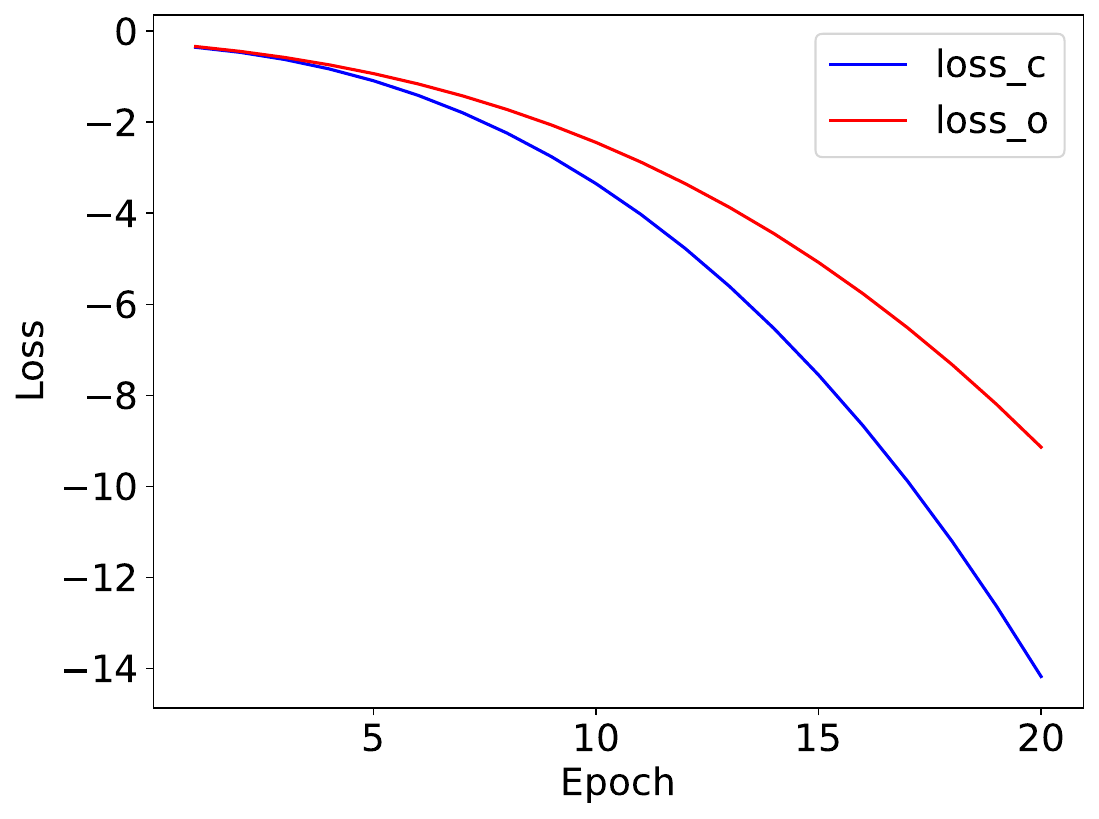}}
      \subfigure[GRACE on Cora]{\includegraphics[width=0.24\textwidth]{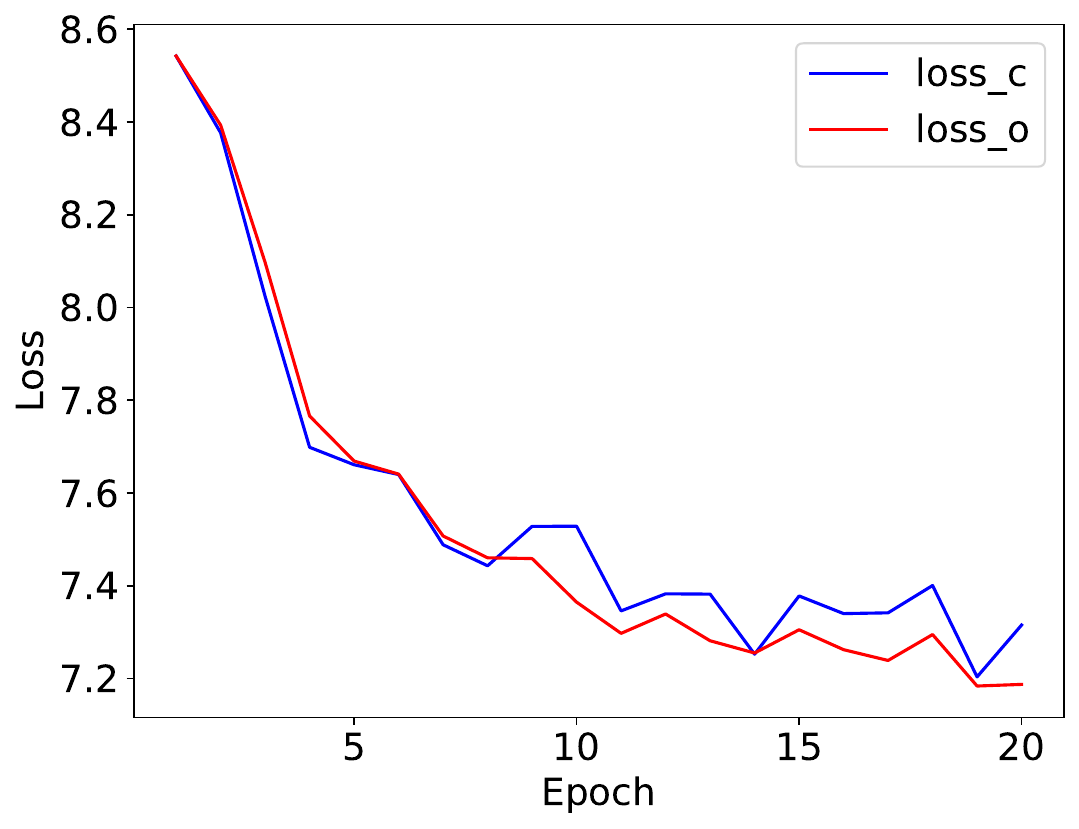}}
        \subfigure[CCA-SSG on Cora]{\includegraphics[width=0.24\textwidth]{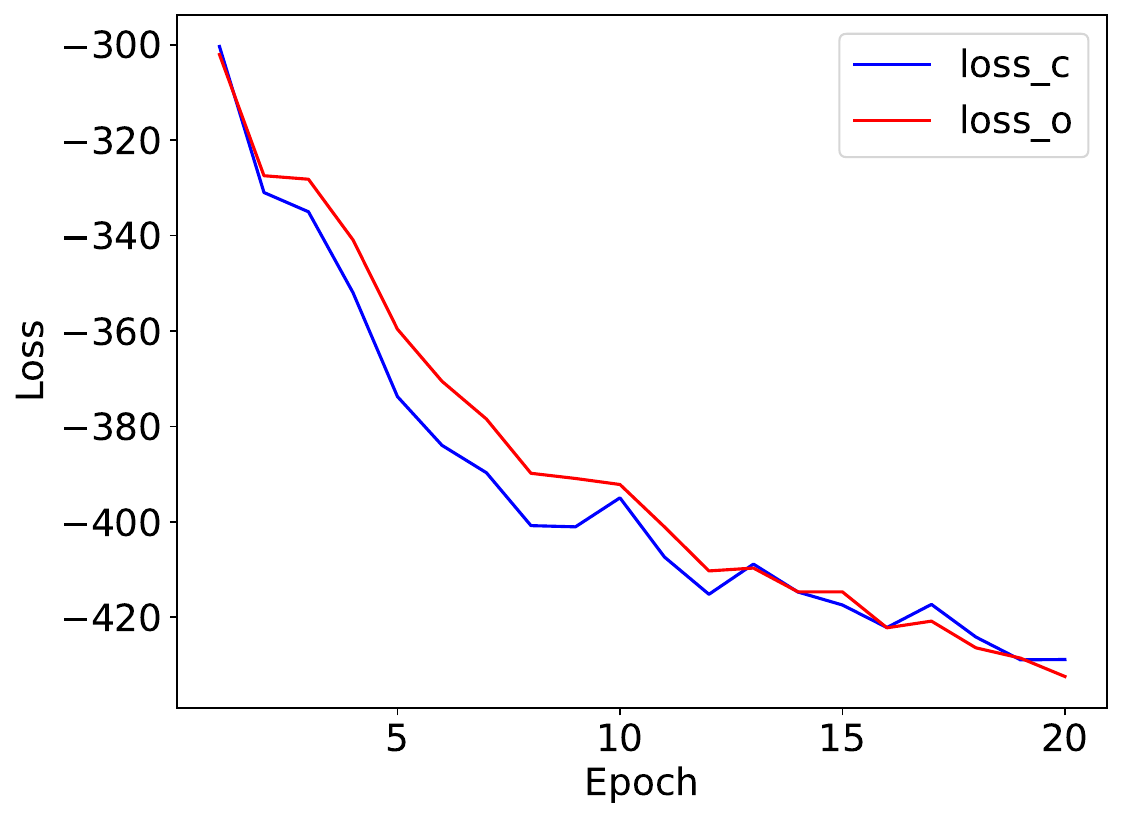}}
          \subfigure[SCE on CiteSeer]{\includegraphics[width=0.24\textwidth]{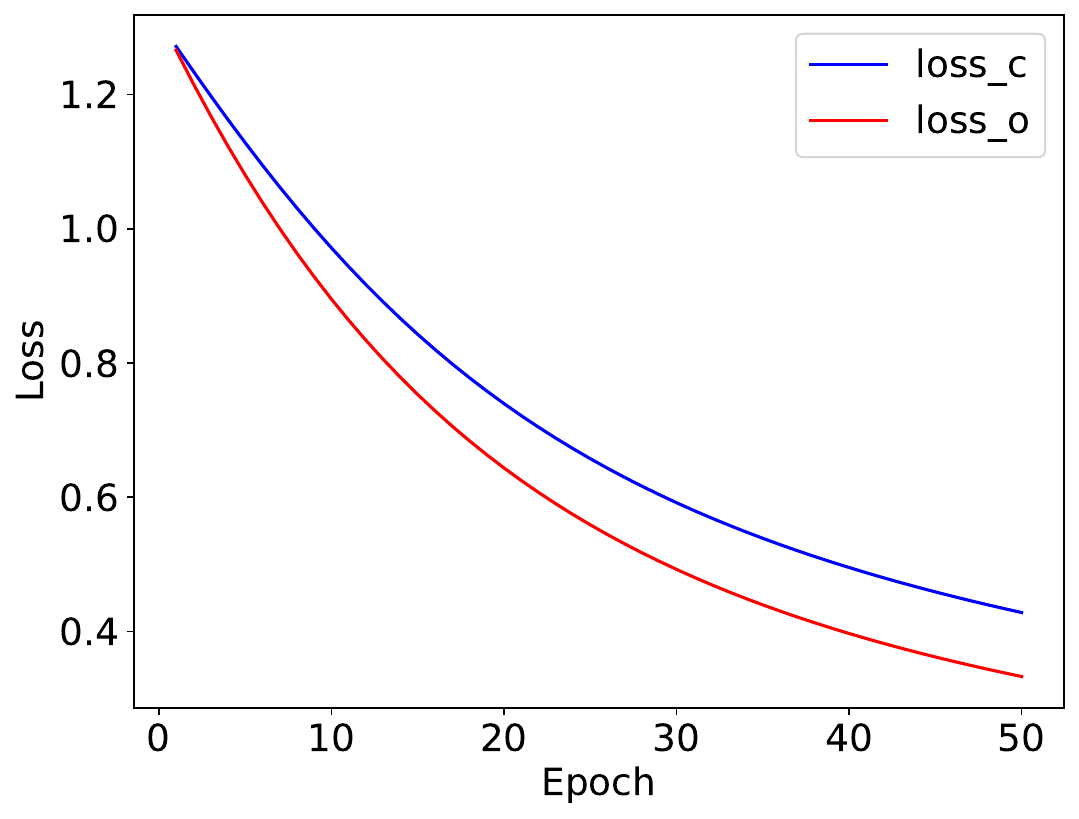}}
    \subfigure[COLES on CiteSeer]{\includegraphics[width=0.24\textwidth]{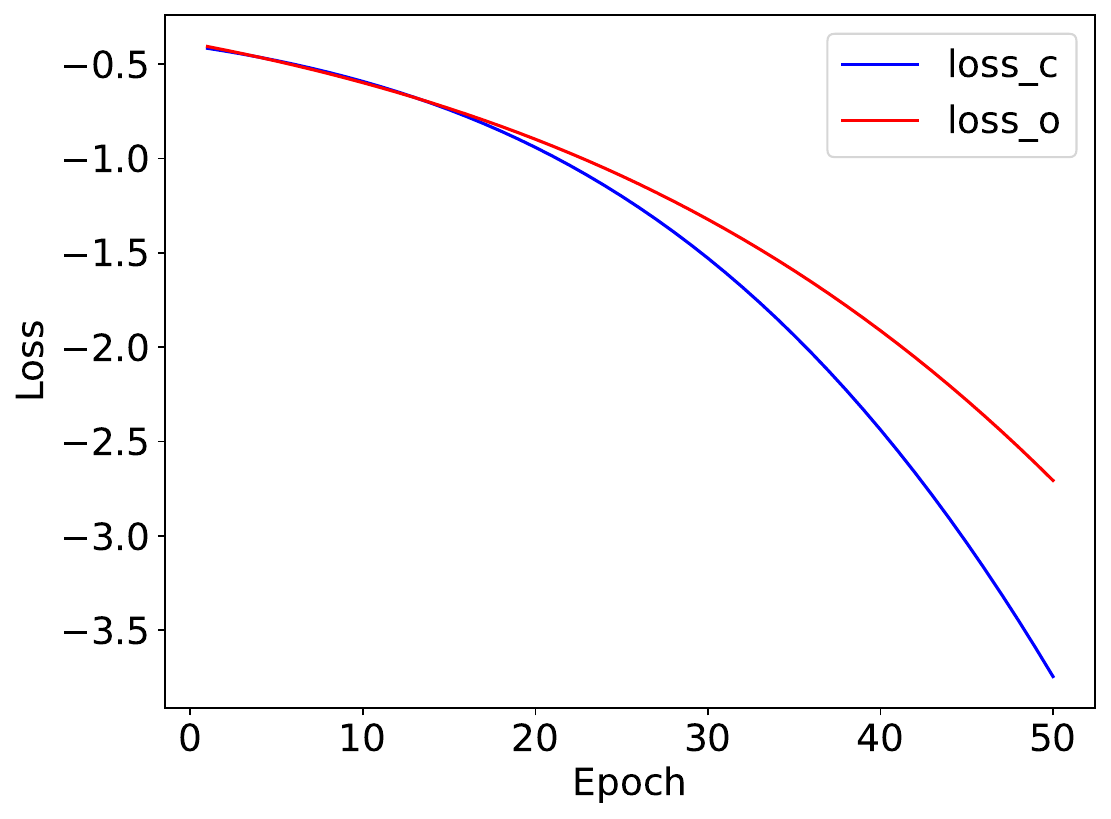}}
      \subfigure[GRACE on CiteSeer]{\includegraphics[width=0.24\textwidth]{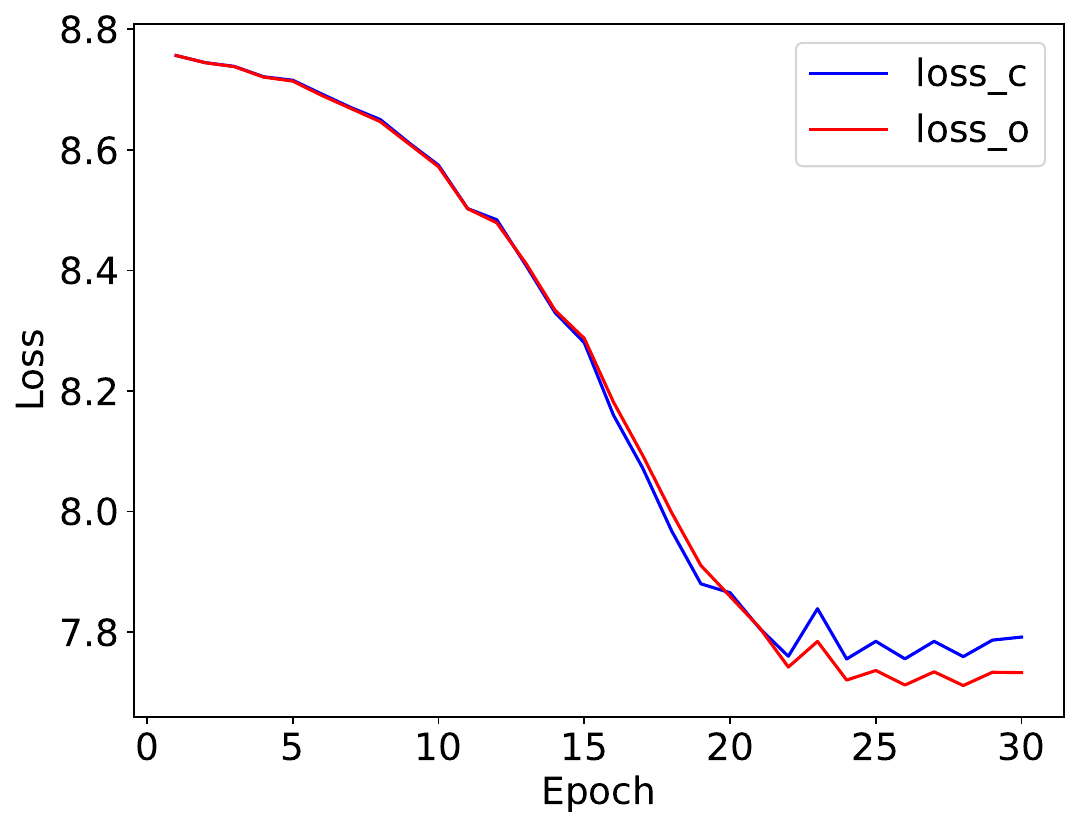}}
        \subfigure[CCA-SSG on CiteSeer]{\includegraphics[width=0.24\textwidth]{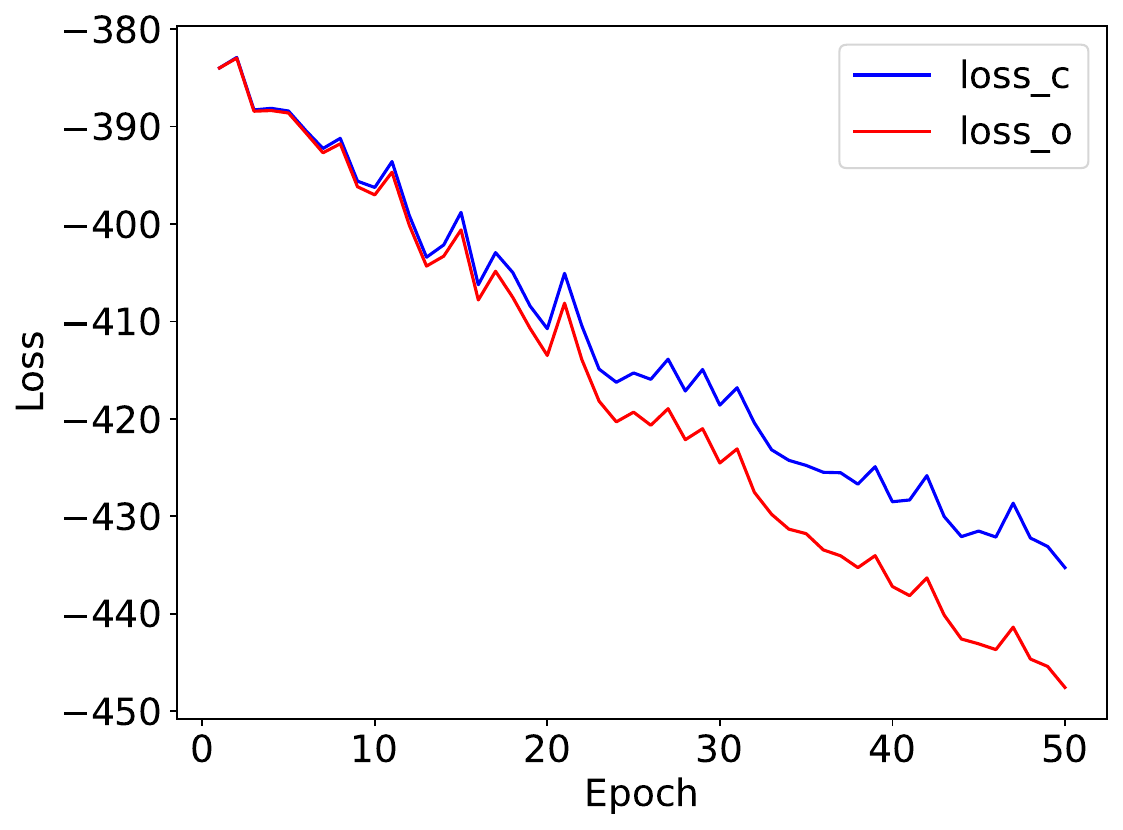}}
	\caption{The trends of the original GCL loss and the loss that computed by \ourmodel-trained parameters. ``loss$\_$o'' is $\mathcal{L}(A,X;W)$ and ``loss$\_$c'' is $\mathcal{L}(A,X;U)$ where $U$ is trained with $\mathcal{L}(X_c;U)$.
 }
 	\label{fig:loss_trend}
\end{figure*}

\begin{figure*}[ht]
\vspace{-3mm}
	\setlength{\abovecaptionskip}{-0.1cm}
	\setlength{\belowcaptionskip}{-0.5cm}
	\centering
  \subfigure[SCE]{\includegraphics[width=0.24\textwidth]{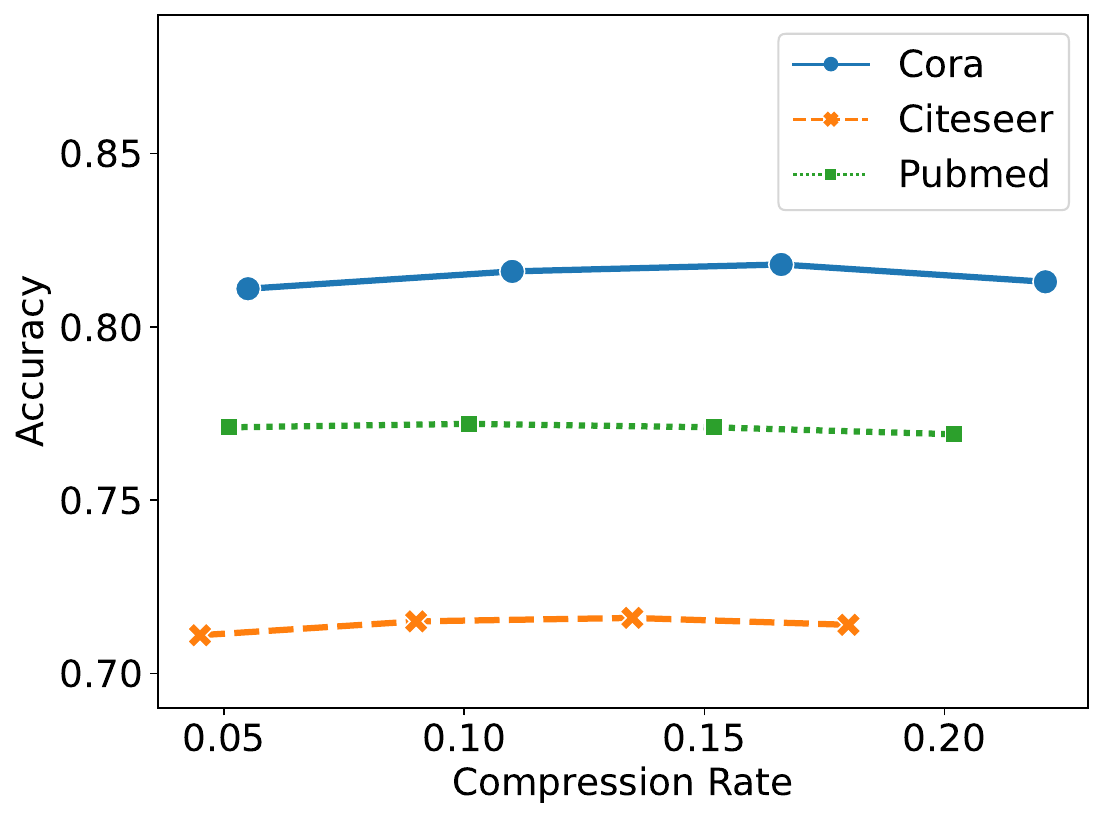}}
    \subfigure[COLES]{\includegraphics[width=0.24\textwidth]{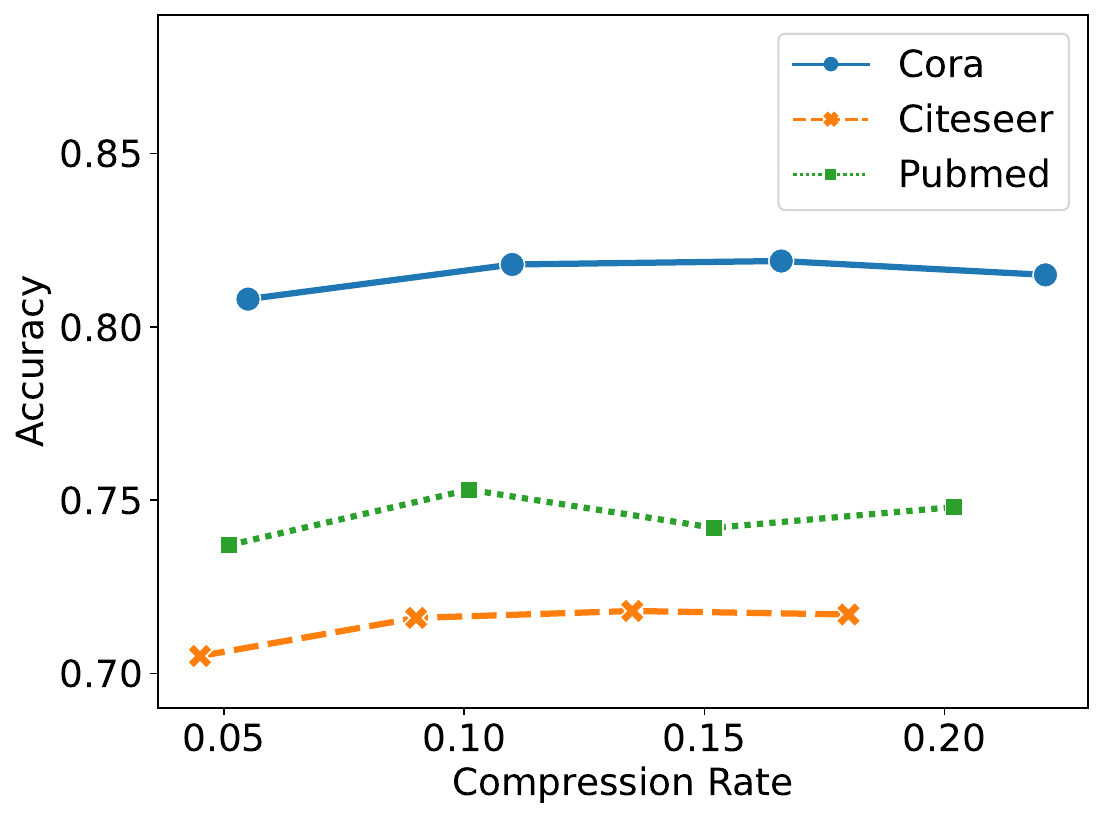}}
      \subfigure[GRACE]{\includegraphics[width=0.24\textwidth]{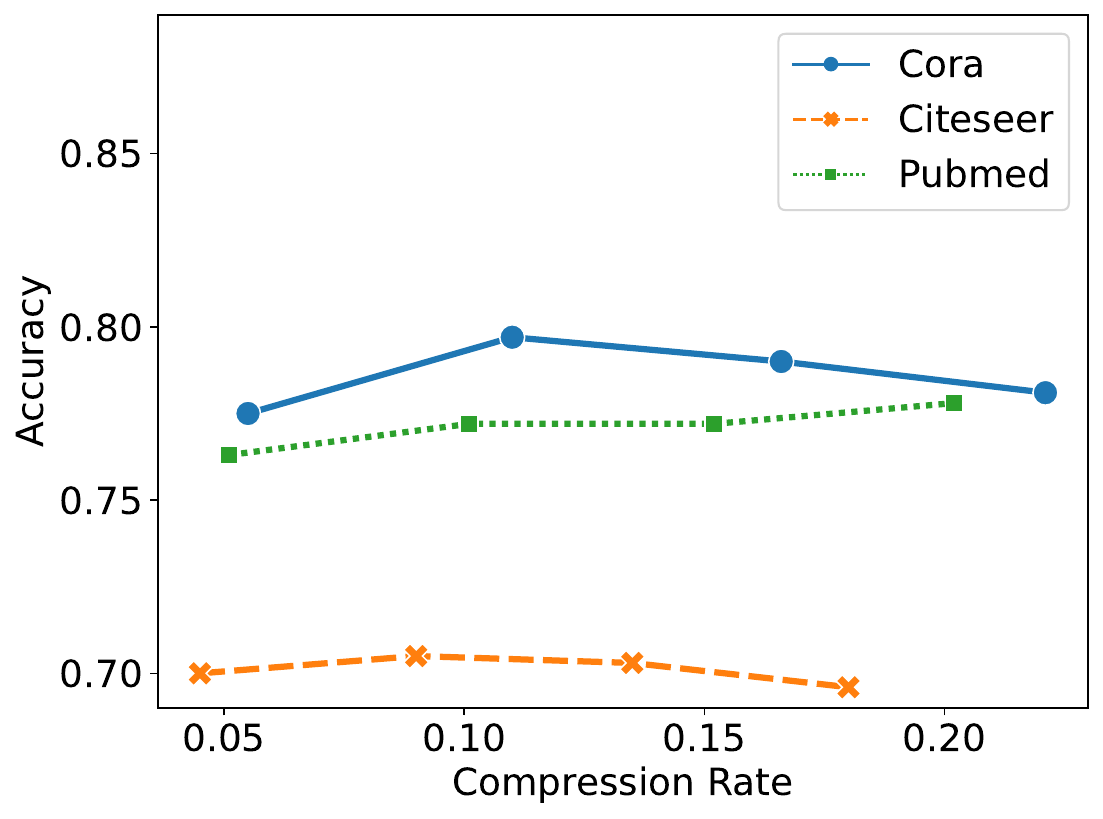}}
        \subfigure[CCA-SSG]{\includegraphics[width=0.24\textwidth]{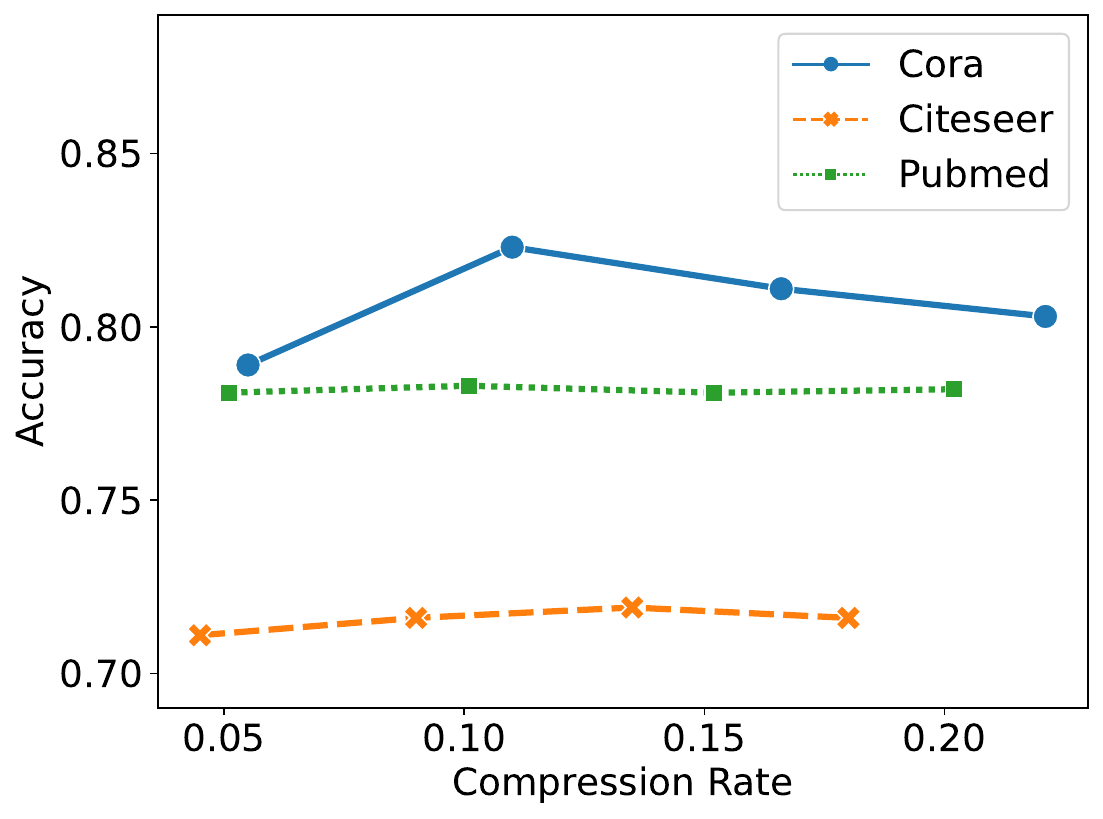}}
	\caption{ The influence of the compression rate on the performance of \ourmodel.}
 	\label{fig:ratio}
\end{figure*}

\section{Conclusion}
In this paper, we introduce \ourmodel, a scalable training framework for GCL. \ourmodel~is driven by a sparse low-rank approximation of the diffusion matrix. In \ourmodel, the message-passing operation is substituted with node compression, leading to substantial reductions in time and memory consumption. Theoretical analysis indicates that \ourmodel~implicitly optimizes the original contrastive loss with fewer resources and is likely to produced a more robust encoder.

\subsubsection*{Acknowledgments}
This work is supported by National Natural Science Foundation of China No.\ U2241212, No.\ 62276066.

\bibliography{iclr2024_conference}

\begin{thebibliography}{46}
\providecommand{\natexlab}[1]{#1}
\providecommand{\url}[1]{\texttt{#1}}
\expandafter\ifx\csname urlstyle\endcsname\relax
  \providecommand{\doi}[1]{doi: #1}\else
  \providecommand{\doi}{doi: \begingroup \urlstyle{rm}\Url}\fi

\bibitem[Alex~Fout \& Ben-Hur(2019)Alex~Fout and Ben-Hur]{alex2017protein}
Basir~Shariat Alex~Fout, Jonathon~Byrd and Asa Ben-Hur.
\newblock Protein interface prediction using graph convolutional networks.
\newblock In \emph{NeurIPS 2017}, 2019.

\bibitem[Belghazi et~al.(2018)Belghazi, Baratin, Rajeshwar, Ozair, Bengio,
  Courville, and Hjelm]{belghazi18mutual}
Mohamed~Ishmael Belghazi, Aristide Baratin, Sai Rajeshwar, Sherjil Ozair,
  Yoshua Bengio, Aaron Courville, and Devon Hjelm.
\newblock Mutual information neural estimation.
\newblock In \emph{International Conference on Machine Learning}, pp.\
  531--540, 2018.

\bibitem[Bojchevski et~al.(2020)Bojchevski, Klicpera, Perozzi, Kapoor, Blais,
  R{\'o}zemberczki, Lukasik, and G{\"u}nnemann]{bojchevski2020scaling}
Aleksandar Bojchevski, Johannes Klicpera, Bryan Perozzi, Amol Kapoor, Martin
  Blais, Benedek R{\'o}zemberczki, Michal Lukasik, and Stephan G{\"u}nnemann.
\newblock Scaling graph neural networks with approximate pagerank.
\newblock In \emph{ACM SIGKDD International Conference on Knowledge Discovery
  \& Data Mining}, pp.\  2464--2473, 2020.

\bibitem[Cai et~al.(2023)Cai, Huang, Xia, and Ren]{cai2023lightgcl}
Xuheng Cai, Chao Huang, Lianghao Xia, and Xubin Ren.
\newblock Lightgcl: Simple yet effective graph contrastive learning for
  recommendation.
\newblock In \emph{International Conference on Learning Representations}, 2023.

\bibitem[Chen et~al.(2018{\natexlab{a}})Chen, Perozzi, Hu, and
  Skiena]{chen2018harp}
Haochen Chen, Bryan Perozzi, Yifan Hu, and Steven Skiena.
\newblock {HARP:} hierarchical representation learning for networks.
\newblock In \emph{Proceedings of AAAI}, pp.\  2127--2134, 2018{\natexlab{a}}.

\bibitem[Chen et~al.(2018{\natexlab{b}})Chen, Zhu, and
  Song]{chen2018stochastic}
Jianfei Chen, Jun Zhu, and Le~Song.
\newblock Stochastic training of graph convolutional networks with variance
  reduction.
\newblock In \emph{International Conference on Machine Learning}, pp.\
  942--950, 2018{\natexlab{b}}.

\bibitem[Chen et~al.(2018{\natexlab{c}})Chen, Ma, and Xiao]{chen2018fastgcn}
Jie Chen, Tengfei Ma, and Cao Xiao.
\newblock {FastGCN}: fast learning with graph convolutional networks via
  importance sampling.
\newblock In \emph{International Conference on Learning Representations},
  2018{\natexlab{c}}.

\bibitem[Chen et~al.(2020{\natexlab{a}})Chen, Wei, Ding, Li, Yuan, Du, and
  Wen]{chen2020scalable}
Ming Chen, Zhewei Wei, Bolin Ding, Yaliang Li, Ye~Yuan, Xiaoyong Du, and
  Ji-Rong Wen.
\newblock Scalable graph neural networks via bidirectional propagation.
\newblock In \emph{Advances in Neural Information Processing Systems},
  2020{\natexlab{a}}.

\bibitem[Chen et~al.(2020{\natexlab{b}})Chen, Wei, Huang, Ding, and
  Li]{chen2020simple}
Ming Chen, Zhewei Wei, Zengfeng Huang, Bolin Ding, and Yaliang Li.
\newblock Simple and deep graph convolutional networks.
\newblock In \emph{International Conference on Machine Learning},
  2020{\natexlab{b}}.

\bibitem[Chiang et~al.(2019)Chiang, Liu, Si, Li, Bengio, and
  Hsieh]{chiang2019cluster}
Wei-Lin Chiang, Xuanqing Liu, Si~Si, Yang Li, Samy Bengio, and Cho-Jui Hsieh.
\newblock Cluster-{GCN}: An efficient algorithm for training deep and large
  graph convolutional networks.
\newblock In \emph{ACM SIGKDD International Conference on Knowledge Discovery
  \& Data Mining}, pp.\  257--266, 2019.

\bibitem[Cong et~al.(2020)Cong, Forsati, Kandemir, and
  Mahdavi]{cong2020minimal}
Weilin Cong, Rana Forsati, Mahmut Kandemir, and Mehrdad Mahdavi.
\newblock Minimal variance sampling with provable guarantees for fast training
  of graph neural networks.
\newblock In \emph{ACM SIGKDD international conference on Knowledge Discovery
  and Data Mining}, pp.\  1393--1403, 2020.

\bibitem[Deng et~al.(2019)Deng, Zhao, Wang, Zhang, and Feng]{deng2019graphzoom}
Chenhui Deng, Zhiqiang Zhao, Yongyu Wang, Zhiru Zhang, and Zhuo Feng.
\newblock Graphzoom: A multi-level spectral approach for accurate and scalable
  graph embedding.
\newblock In \emph{International Conference on Learning Representations}, 2019.

\bibitem[Fang et~al.(2023)Fang, Xiao, Wang, Xu, Yang, and
  Yang]{fang2023dropmessage}
Taoran Fang, Zhiqing Xiao, Chunping Wang, Jiarong Xu, Xuan Yang, and Yang Yang.
\newblock Dropmessage: Unifying random dropping for graph neural networks.
\newblock In \emph{Proceedings of the AAAI Conference on Artificial
  Intelligence}, 2023.

\bibitem[Grover \& Leskovec(2016)Grover and Leskovec]{DBLP:conf/kdd/GroverL16}
Aditya Grover and Jure Leskovec.
\newblock node2vec: Scalable feature learning for networks.
\newblock In \emph{Proceedings of International Conference on Knowledge
  Discovery and Data Mining}, pp.\  855--864, 2016.

\bibitem[Hamilton et~al.(2017)Hamilton, Ying, and
  Leskovec]{hamilton2017inductive}
Will Hamilton, Zhitao Ying, and Jure Leskovec.
\newblock Inductive representation learning on large graphs.
\newblock In \emph{Advances in Neural Information Processing Systems}, pp.\
  1024--1034, 2017.

\bibitem[Hassani \& Ahmadi(2020)Hassani and Ahmadi]{hassani2020contrastive}
Kaveh Hassani and Amir Hosein~Khas Ahmadi.
\newblock Contrastive multi-view representation learning on graphs.
\newblock In \emph{International Conference on Machine Learning}, pp.\
  4116--4126, 2020.

\bibitem[Hjelm et~al.(2019)Hjelm, Fedorov, Lavoie-Marchildon, Grewal, Bachman,
  Trischler, and Bengio]{hjelm2018learning}
R~Devon Hjelm, Alex Fedorov, Samuel Lavoie-Marchildon, Karan Grewal, Phil
  Bachman, Adam Trischler, and Yoshua Bengio.
\newblock Learning deep representations by mutual information estimation and
  maximization.
\newblock In \emph{International Conference on Learning Representations}, 2019.

\bibitem[Hu et~al.(2020)Hu, Fey, Zitnik, Dong, Ren, Liu, Catasta, and
  Leskovec]{hu2020ogb}
Weihua Hu, Matthias Fey, Marinka Zitnik, Yuxiao Dong, Hongyu Ren, Bowen Liu,
  Michele Catasta, and Jure Leskovec.
\newblock Open graph benchmark: Datasets for machine learning on graphs.
\newblock In \emph{Advances in neural information processing systems}, pp.\
  22118--22133, 2020.

\bibitem[Huang et~al.(2021)Huang, Zhang, Xi, Liu, and Zhou]{huang2021scaling}
Zengfeng Huang, Shengzhong Zhang, Chong Xi, Tang Liu, and Min Zhou.
\newblock Scaling up graph neural networks via graph coarsening.
\newblock In \emph{ACM SIGKDD international conference on Knowledge Discovery
  and Data Mining}, pp.\  675--684, 2021.

\bibitem[Kipf \& Welling(2017)Kipf and Welling]{kipf2016semi}
Thomas~N Kipf and Max Welling.
\newblock Semi-supervised classification with graph convolutional networks.
\newblock In \emph{International Conference on Learning Representations}, 2017.

\bibitem[Li et~al.(2023)Li, Sun, Wu, Zhu, Chen, and Zheng]{blockgcl}
Jintang Li, Wangbin Sun, Ruofan Wu, Yuchang Zhu, Liang Chen, and Zibin Zheng.
\newblock Scaling up, scaling deep: Blockwise graph contrastive learning.
\newblock \emph{Arxiv}, 2023.

\bibitem[Liang et~al.(2018)Liang, Gurukar, and Parthasarathy]{liang2018mile}
Jiongqian Liang, Saket Gurukar, and Srinivasan Parthasarathy.
\newblock Mile: A multi-level framework for scalable graph embedding.
\newblock \emph{arXiv preprint arXiv:1802.09612}, 2018.

\bibitem[Liu et~al.(2020)Liu, Gao, and Ji]{liu2020towards}
Meng Liu, Hongyang Gao, and Shuiwang Ji.
\newblock Towards deeper graph neural networks.
\newblock In \emph{ACM SIGKDD International Conference on Knowledge Discovery
  and Data Mining}, 2020.

\bibitem[Ma et~al.(2023)Ma, Yang, Yang, Chen, and Cheng]{contrastreg}
Kaili Ma, Haochen Yang, Han Yang, Yongqiang Chen, and James Cheng.
\newblock Calibrating and improving graph contrastive learning.
\newblock \emph{Transactions on Machine Learning Research}, 2023.

\bibitem[Markowitz et~al.(2021)Markowitz, Balasubramanian, Mirtaheri,
  Abu-El-Haija, Perozzi, Ver~Steeg, and Galstyan]{markowitz2021graph}
Elan~Sopher Markowitz, Keshav Balasubramanian, Mehrnoosh Mirtaheri, Sami
  Abu-El-Haija, Bryan Perozzi, Greg Ver~Steeg, and Aram Galstyan.
\newblock Graph traversal with tensor functionals: A meta-algorithm for
  scalable learning.
\newblock In \emph{International Conference on Learning Representations}, 2021.

\bibitem[Perozzi et~al.(2014)Perozzi, Al-Rfou, and
  Skiena]{Perozzi:2014:DOL:2623330.2623732}
Bryan Perozzi, Rami Al-Rfou, and Steven Skiena.
\newblock Deepwalk: Online learning of social representations.
\newblock In \emph{ACM SIGKDD international conference on Knowledge Discovery
  and Data Mining}, 2014.

\bibitem[Ramezani et~al.(2020)Ramezani, Cong, Mahdavi, Sivasubramaniam, and
  Kandemir]{ramezani2020gcn}
Morteza Ramezani, Weilin Cong, Mehrdad Mahdavi, Anand Sivasubramaniam, and
  Mahmut Kandemir.
\newblock Gcn meets gpu: Decoupling “when to sample” from “how to
  sample”.
\newblock In \emph{Advances in Neural Information Processing Systems}, 2020.

\bibitem[Rong et~al.(2020)Rong, Huang, Xu, and Huang]{rong2019dropedge}
Yu~Rong, Wenbing Huang, Tingyang Xu, and Junzhou Huang.
\newblock Dropedge: Towards deep graph convolutional networks on node
  classification.
\newblock In \emph{International Conference on Learning Representations}, 2020.

\bibitem[Velickovic et~al.(2018)Velickovic, Cucurull, Casanova, Romero, Lio,
  and Bengio]{velickovic2017graph}
Petar Velickovic, Guillem Cucurull, Arantxa Casanova, Adriana Romero, Pietro
  Lio, and Yoshua Bengio.
\newblock Graph attention networks.
\newblock In \emph{International Conference on Learning Representations}, 2018.

\bibitem[Veli{\v{c}}kovi{\'c} et~al.(2018)Veli{\v{c}}kovi{\'c}, Fedus,
  Hamilton, Li{\`o}, Bengio, and Hjelm]{velivckovic2018deep}
Petar Veli{\v{c}}kovi{\'c}, William Fedus, William~L Hamilton, Pietro Li{\`o},
  Yoshua Bengio, and R~Devon Hjelm.
\newblock Deep graph infomax.
\newblock In \emph{International Conference on Learning Representations}, 2018.

\bibitem[Wang et~al.(2023)Wang, Zhang, Zhu, Huang, Kawaguchi, and Xiao]{spgcl}
Haonan Wang, Jieyu Zhang, Qi~Zhu, Wei Huang, Kenji Kawaguchi, and Xiaokui Xiao.
\newblock Single-pass contrastive learning can work for both homophilic and
  heterophilic graph.
\newblock \emph{Transactions on Machine Learning Research}, 2023.

\bibitem[Wang et~al.(2020)Wang, Ma, Wang, Jin, Wang, Tang, Jia, and
  Yu]{Wang2020traffic}
Xiaoyang Wang, Yao Ma, Yiqi Wang, Wei Jin, Xin Wang, Jiliang Tang, Caiyan Jia,
  and Jian Yu.
\newblock Traffic flow prediction via spatial temporal graph neural network.
\newblock In \emph{ACM SIGKDD international conference on Knowledge Discovery
  and Data Mining}, pp.\  1082--1092, 2020.

\bibitem[Wang et~al.(2022)Wang, Zhou, Miao, Liu, and Wang]{wang2022adagcl}
Yili Wang, Kaixiong Zhou, Rui Miao, Ninghao Liu, and Xin Wang.
\newblock Adagcl: Adaptive subgraph contrastive learning to generalize
  large-scale graph training.
\newblock In \emph{ACM International Conference on Information and Knowledge
  Management}, 2022.

\bibitem[Wu et~al.(2019)Wu, Souza, Zhang, Fifty, Yu, and
  Weinberger]{wu2019simplifying}
Felix Wu, Amauri Souza, Tianyi Zhang, Christopher Fifty, Tao Yu, and Kilian
  Weinberger.
\newblock Simplifying graph convolutional networks.
\newblock In \emph{International Conference on Machine Learning}, pp.\
  6861--6871, 2019.

\bibitem[You et~al.(2020)You, Chen, Sui, Chen, Wang, and Shen]{you2020graph}
Yuning You, Tianlong Chen, Yongduo Sui, Ting Chen, Zhangyang Wang, and Yang
  Shen.
\newblock Graph contrastive learning with augmentations.
\newblock In \emph{Advances in neural information processing systems}, pp.\
  5812--5823, 2020.

\bibitem[Zeng et~al.(2020)Zeng, Zhou, Srivastava, Kannan, and
  Prasanna]{zeng2020graphsaint}
Hanqing Zeng, Hongkuan Zhou, Ajitesh Srivastava, Rajgopal Kannan, and Viktor~K.
  Prasanna.
\newblock Graphsaint: Graph sampling based inductive learning method.
\newblock In \emph{International Conference on Learning Representations}, 2020.

\bibitem[Zeng et~al.(2021)Zeng, Zhang, Xia, Srivastava, Malevich, Kannan,
  Prasanna, Jin, and Chen]{zeng2021decoupling}
Hanqing Zeng, Muhan Zhang, Yinglong Xia, Ajitesh Srivastava, Andrey Malevich,
  Rajgopal Kannan, Viktor Prasanna, Long Jin, and Ren Chen.
\newblock Decoupling the depth and scope of graph neural networks.
\newblock In \emph{Advances in Neural Information Processing Systems}, 2021.

\bibitem[Zhang et~al.(2021{\natexlab{a}})Zhang, Li, Yuan, Xu, Cao, Xu, and
  Jin]{zhang2021cvp}
Guozhen Zhang, Yong Li, Yuan Yuan, Fengli Xu, Hancheng Cao, Yujian Xu, and
  Depeng Jin.
\newblock Community value prediction in social e-commerce.
\newblock In Jure Leskovec, Marko Grobelnik, Marc Najork, Jie Tang, and Leila
  Zia (eds.), \emph{The Web Conference}, pp.\  2958--2967, 2021{\natexlab{a}}.

\bibitem[Zhang et~al.(2021{\natexlab{b}})Zhang, Wu, Yan, Wipf, and
  Yu]{zhang2021ccassg}
Hengrui Zhang, Qitian Wu, Junchi Yan, David Wipf, and Philip~S. Yu.
\newblock From canonical correlation analysis to self-supervised graph neural
  networks.
\newblock In \emph{Advances in Neural Information Processing Systems},
  2021{\natexlab{b}}.

\bibitem[Zhang et~al.(2020)Zhang, Huang, Zhou, and Zhou]{zhang2020sce}
Shengzhong Zhang, Zengfeng Huang, Haicang Zhou, and Ziang Zhou.
\newblock {SCE:} scalable network embedding from sparsest cut.
\newblock In \emph{ACM SIGKDD international conference on Knowledge Discovery
  and Data Mining}, pp.\  257--265, 2020.

\bibitem[Zheng et~al.(2022)Zheng, Pan, Lee, Zheng, and Yu]{zheng2022ggd}
Yizhen Zheng, Shirui Pan, Vincent C.~S. Lee, Yu~Zheng, and Philip~S. Yu.
\newblock Rethinking and scaling up graph contrastive learning: An extremely
  efficient approach with group discrimination.
\newblock In \emph{Advances in Neural Information Processing Systems}, 2022.

\bibitem[Zhu \& Koniusz(2021)Zhu and Koniusz]{zhu2021ssgc}
Hao Zhu and Piotr Koniusz.
\newblock Simple spectral graph convolution.
\newblock In \emph{International Conference on Learning Representations}, 2021.

\bibitem[Zhu et~al.(2021{\natexlab{a}})Zhu, Sun, and Koniusz]{zhu2021coles}
Hao Zhu, Ke~Sun, and Peter Koniusz.
\newblock Contrastive laplacian eigenmaps.
\newblock In \emph{Advances in neural information processing systems}, pp.\
  5682--5695, 2021{\natexlab{a}}.

\bibitem[Zhu et~al.(2020)Zhu, Xu, Yu, Liu, Wu, and Wang]{zhu2020grace}
Yanqiao Zhu, Yichen Xu, Feng Yu, Qiang Liu, Shu Wu, and Liang Wang.
\newblock Deep graph contrastive representation learning.
\newblock In \emph{International Conference on Machine Learning Workshop on
  Graph Representation Learning and Beyond}, 2020.

\bibitem[Zhu et~al.(2021{\natexlab{b}})Zhu, Xu, Yu, Liu, Wu, and
  Wang]{zhu2021gca}
Yanqiao Zhu, Yichen Xu, Feng Yu, Qiang Liu, Shu Wu, and Liang Wang.
\newblock Graph contrastive learning with adaptive augmentation.
\newblock In \emph{The Web Conference}, pp.\  2069--2080, 2021{\natexlab{b}}.

\bibitem[Zou et~al.(2019)Zou, Hu, Wang, Jiang, Sun, and Gu]{zou2019layer}
Difan Zou, Ziniu Hu, Yewen Wang, Song Jiang, Yizhou Sun, and Quanquan Gu.
\newblock Layer-dependent importance sampling for training deep and large graph
  convolutional networks.
\newblock In \emph{Advances in neural information processing systems}, 2019.

\end{thebibliography}
\bibliographystyle{iclr2024_conference}

\appendix
\newpage
\renewcommand{\thetheorem}{4.\arabic{theorem}}

\section{Proof details}

\subsection{The non-linear extension for Equation \ref{eq:opt}}\label{appendix:nonlinear}
We explain how to extend the results to non-linear deep models below. 
 Equation \ref{eq:opt} provides the motivation of structural compression on a linear GNN (which can also be considered as an approximation to one layer in a multi-layer non-linear GNN). The analysis can be extended to non-linear deep GNNs. For instance, given a two-layer non-linear GCN $\sigma(\hat{A}\sigma(\hat{A}XW_1)W_2)$, we first approximate $\Tilde{A}$ by $P' P^T$, then the whole GCN can be approximated as
\begin{equation}
    \begin{split}
        \sigma(P' P^T\sigma(P' P^TXW_1)W_2)=\sigma(P' P^TP'\sigma(P^TXW_1)W_2)=P'\sigma(\sigma(P^TXW_1)W_2).
    \end{split}
\end{equation}
The first equality holds because $P'$ is a partition matrix and the last equality follows from the fact that $P^TP' = I$. Therefore, our analysis provides theoretical justifications of using StructComp as a substitute for non-linear deep GNNs.

\subsection{Proof for Theorem~\ref{thm:partition} }\label{appendix:equivalence}

\paragraph{Single-view Graph Contrastive Learning} We view adjacent nodes as positive pairs and non-adjacent nodes as negative pairs. In a coarse graph of supernodes, we define a pair of supernodes as positive if there exist two nodes within each supernode that are adjacent in the original graph.
Let $f:\mathcal{X}\rightarrow \R^{d'}$ be a feature mapping. We are interested in a variant of the contrastive loss function that aims to minimize the distance between the features of positive pairs $v,v^+$.

\begin{align}
\label{eq:loss}
    \mathcal{L}(f)= \E_{v,v^+} \|f(v)-f(v^+)\|_2.
\end{align}

\paragraph{Graph Partition}
For a graph $G=(V,E,X)$, we consider the Erd\H{o}s-Rényi model, denoted as $G(n,p)$, where edges between $n$ vertices are included independently with probability $p$. We compute a partition $\mathcal{P}$ on the nodes. We denote $\mathcal{P}=\{S_1,\cdots,S_{n'}\}$, such that $V = \bigcup_{j=1}^{n'} S_j$ and $S_{j_1}\bigcap S_{j_2}=\emptyset$ for $j_1\neq j_2$. 
We define a partition $\mathcal{P}$ to be an \textit{even partition} if each subset has the same size $|S_1|=|S_2|=\cdots=|S_{n'}|$.
For each vertex $v$, we denote $S(v)\in \mathcal{P}$ as the subset that $v$ belongs to. We denote $N(v):=\{s: (s,v)\in E \text{ or } (v,s)\in E\}$ as the set of neighbors of $v$. Denote $M(v):=N(v)\backslash S(v)$ be the set of $v$'s neighbors that are not in the subset with $v$.

For the partition $\mathcal{P}$, we denote $P'\in \{0,1\}^{n \times n'}$ as its corresponding partition matrix, where $P'_{ij}=1$ if the node $v_i$ belongs to the subset $S_j$. Let $R=A-P' P'^{T}$ be the partition remainder. By definition, $R_{ij}=1$ if node $j\in M(i)$. Here we define the \textit{partition loss} as
\[
\mathcal{L}_{\text{partition}} = \|R\|_F:=\|A-P' P'^{T}\|_F ,\text{ where }\|\cdot\|_F \text{ denotes the Frobenius norm.  }
\]

Now we are ready to prove Theorem~\ref{thm:partition}.

\begin{theorem}
    For the random graph $G(n,p)$ from Erd\H{o}s-Rényi model, we
    construct an
    even partition $\mathcal{P}=\{S_1,\cdots,S_{n'}\}$. Let  $f_G(X)=AXW$ be a feature mapping in the original graph and $f_\mathcal{P}(X)=P'^{T} XW$ as a linear mapping for the mixed nodes, where $W\in\R^{d\times d'}$. Then by conducting single-view contrastive learning, the contrastive loss for the original graph, denoted as $\mathcal{L}_G(W)$, can be approximated as the sum of the compressed contrastive loss, $\mathcal{L}_\mathcal{P}(W)$, and a term related to the low-rank approximation. Assume the features are bounded by $S_X := \max_i\|X_i\|_2$, we have
    \[
    \lvert\mathcal{L}_G(W) - \mathcal{L}_\mathcal{P}(W) \rvert 
    \leq \|A-P'P'^{T}\|_F S_X\|W\|_2.
    \]
\end{theorem}

\begin{proof}
    We denote the transpose of $i$-th row of $X$ as $X_{i}$. In the original graph, by the definition of adjacent matrix $A$, 
the embedding of each node $v_i$ is
\[
f_{G,W}(v_i) = (AXW)_{i} = \sum_{v_s\in N(v_i)} W^T X_s.
\]
For the compressed loss, each mixed node corresponds to a subset $S_i$ of the partition $\mathcal{P}$. We denote $B(S_i):=\{S_j\in \mathcal{P}: \exists u\in S_i,v\in S_j, s.t. (u,v)\in E\}$ as the set of mixed nodes that are connected to $S_i$.
Then the embedding of mixed node $S_i$ can be written as
\[
f_{\mathcal{P},W}(S_i) = (P'^{T} X W)_{i} 
\]
For a positive pair $u,v$ in the original graph, we can measure their difference in the feature subspace as follows.
\begin{align*}
     f_{G,W}(u)- f_{G,W}(v)
     =&  (AXW)_{u} - (AXW)_{v}\\
     = &W^T(\sum_{v_i\in S(u)} X_i - \sum_{v_i \in S(v)} X_i + \sum_{v_i \in M(u)}  X_i - \sum_{v_i\in M(v)} X_i )\\
    =&(P'^{T} XW)_{S(u)} -(P'^{T} XW)_{S(v)}  + (RXW)_{u}-(RXW)_{v}.
\end{align*}
For a positive pair $(S_i,S_j)$ in the compressed loss, their difference in the feature subspace is
\begin{align*}
    f_{\mathcal{P},W}(S_i)-f_{\mathcal{P},W}(S_j) 
    = (P'^{T} XW)_i - (P'^{T} XW)_j.
\end{align*}
For any $i,j\in[n]$, $(v_i,v_j)$ is an edge in $G$ independently with probability $p$. We denote $\mathrm{pos}$ as the set of positive pairs in the compressed loss.
Then by calculating contrastive loss on both graphs, we have
\begin{align*}
    &\mathcal{L}_G(W) - \mathcal{L}_{\mathcal{P}}(W)\\
    = &\mathop{\E}\limits_{(u,v)\in E}\| f_{G,W}(u)-f_{G,W}(v)\|_2
    - \mathop{\E}\limits_{(S_i,S_j) \in\text{ pos}}  \|f_{\mathcal{P},W}(S_i)-f_{\mathcal{P},W}(S_j) \|_2\\
    =&\mathop{\E}\limits_{(u,v)\in E} \|(P'^{T} XW)_{S(u)} -(P'^{T} XW)_{S(v)}  + (RXW)_{u}-(RXW)_{v}\|_2\\
    &- \mathop{\E}\limits_{(S_i,S_j) \in\text{ pos}}  \| (P'^{T} XW)_i - (P'^{T} XW)_j \|_2\\
    \leq & \mathop{\E}\limits_{(u,v)\in E} \|(P'^{T} XW)_{S(u)} -(P'^{T} XW)_{S(v)} \|_2 -\mathop{\E}\limits_{(S_i,S_j) \in\text{ pos}}  \| (P'^{T} XW)_i - (P'^{T} XW)_j \|_2 \\
    & +\mathop{\E}\limits_{(u,v)\in E} \|(RXW)_{u}-(RXW)_{v}\|_2\\
    =& \mathop{\E}\limits_{(u,v)\in E} \|(RXW)_{u}-(RXW)_{v}\|_2.
\end{align*}
The last step holds because the partition is even. On the other hand,
\begin{align*}
    &\mathcal{L}_{\mathcal{P}}(W) - \mathcal{L}_G(W)\\
    =& \mathop{\E}\limits_{(S_i,S_j) \in\text{ pos}}  \| (P'^{T} XW)_i - (P'^{T} XW)_j \|_2\\
    &- \mathop{\E}\limits_{(u,v)\in E} \|(P'^{T} XW)_{D(u)} -(P'^{T} XW)_{D(v)}  + (RXW)_{u}-(RXW)_{v}\|_2\\
    \leq & \mathop{\E}\limits_{(S_i,S_j) \in\text{ pos}}  \| (P'^{T} XW)_i - (P'^{T} XW)_j \|_2\\
    &-  \mathop{\E}\limits_{(u,v)\in E} \|(P'^{T} XW)_{D(u)} -(P'^{T} XW)_{D(v)} \|_2
    + \mathop{\E}\limits_{(u,v)\in E} \|(RXW)_{u}-(RXW)_{v}\|_2\\
    =&\mathop{\E}\limits_{(u,v)\in E} \|(RXW)_{u}-(RXW)_{v}\|_2.
\end{align*}
We combine both inequalities. Denote $\eta = \|A-P' P'^{T}\|_F$. Then we have
\begin{align*}
    \lvert\mathcal{L}_G(W) - \mathcal{L}_{\mathcal{P}}(W) \rvert 
    \leq& \mathop{\E}\limits_{(u,v)\in E} \|(RXW)_{u}-(RXW)_{v}\|_2\\
    \leq& \|\delta^T XW\|_2 \quad\qquad\rhd \delta \in \R^n \text{ is an }\eta \text{ sparse vector with }0,\pm 1\text{ entries}\\
    \leq &\sum_{i\in[n]:\delta_i\neq 0} \|X_iW\|_2\\
    \leq & \eta S_X\|W\|_2\\
    =& \|A-P' P'^{T}\|_F S_X\|W\|_2.
\end{align*}

\end{proof}

\subsection{An upper bound of the approximation gap without ER graph}\label{appendix:nonrandom}
We give an extra analysis on arbitrary graphs. For non-random graphs, the approximation gap of losses is simply bounded by the Equation \ref{eq:opt}. Suppose the loss $\mathcal{L}$ is $L$-Lipschitz continuous,
\begin{equation}
    \begin{split}
        |\mathcal{L}(P' P^TXW)-\mathcal{L}(\hat{A}^kXW)|\leq L \underbrace{\Vert P' P^T - \hat{A}^k\Vert}_\text{Equation \ref{eq:opt}}  \Vert X\Vert  \Vert  W\Vert .
    \end{split}
\end{equation}
And for a spectral contrastive loss $\mathcal{L}_{\text{spec}}$, assume the graph partition are even, we have:
\begin{equation}
    \begin{split}
        \mathcal{L}_{\text{spec}}(P^TXW)&=-\frac{2}{n}\sum_{i=1}^n e^T_{1,i}e_{2,i}+\frac{1}{n^2}\sum_{i=1}^n \sum_{j=1}^n (e^T_{1,i} e_{2,j})^2\\
        &=-\frac{2}{n}\sum_{k=1}^{n'}\sum_{i\in S_k}e^T_{1,i}e_{2,i}+\frac{1}{n^2}\sum_{i=1}^n \sum_{l=1}^{n'} \sum_{j\in S_l} (e^T_{1,i} e_{2,j})^2\\
        &=-\frac{2}{n'}\sum_{k=1}^{n'}E^T_{1,i}E_{2,i}+\frac{1}{nn'}\sum_{i=1}^n \sum_{l=1}^{n'}  (e^T_{1,i} E_{2,j})^2\\
        &=-\frac{2}{n'}\sum_{k=1}^{n'}E^T_{1,i}E_{2,i}+\frac{1}{nn'}\sum_{k=1}^{n'}\sum_{i\in S_k} \sum_{l=1}^{n'}  (e^T_{1,i} E_{2,j})^2\\
        &=-\frac{2}{n'}\sum_{k=1}^{n'}E^T_{1,i}E_{2,i}+\frac{1}{n'^2}\sum_{k=1}^{n'}\sum_{l=1}^{n'}  (E^T_{1,i} E_{2,j})^2=\mathcal{L}_{\text{spec}}(P' P^TXW),
    \end{split}
\end{equation}
where $e_{1,i}$ denotes the representations of a recovered node and $E^T_{1,i}$ denotes the representations of a compressed node. The above analysis shows that our approximation is reasonable for fixed graphs.

\subsection{Proof for Theorem \ref{the:regulation}} \label{appendix:regularization}
\begin{theorem}
    Consider a no-augmentation InfoNCE loss,
    \begin{equation}
        \mathcal{L}_{\mathrm{InfoNCE}}=\sum_i \sum_{j\in \mathrm{pos}(i)}[h_i^T h_j] + \sum_i \sum_{j\in \mathrm{neg}(i)} [\log (e^{h_i^T h_i}+e^{h_i^T h_j})].
    \end{equation}
    Optimizing the expectation of this with augmentation $\mathbb{E}[\Tilde{\mathcal{L}}_{\mathrm{InfoNCE}}]$ introduce an additional regularization term, i.e.,
    \begin{equation}
         \mathbb{E}[\Tilde{\mathcal{L}}_{\mathrm{InfoNCE}}]=\mathcal{L}_{\mathrm{InfoNCE}}+\frac{1}{2}\sum_i \sum_{j\in \mathrm{neg}(i)}\phi(h_i,h_j)\mathrm{Var}(\Tilde{h}_i),
    \end{equation}
where $\phi(h_i,h_j)=\frac{(e^{h_i^2}h_i^2+e^{h_i h_j}h_j^2)(e^{h_i^2}+e^{h_i h_j})-(e^{h_i^2}h_i+e^{h_i h_j}h_j)^2}{2(e^{h_i^2}+e^{h_i h_j})^2}$.
\end{theorem}
\begin{proof}
    
\begin{equation}
    \mathbb{E}[\Tilde{\mathcal{L}}_{\mathrm{InfoNCE}}]= \sum_i \sum_{j\in \mathrm{pos}(i)}[h_i^Th_j]  + \mathbb{E}[\Delta_1] + \sum_i \sum_{j\in \mathrm{neg}(i)} [\log (e^{h_i^T h_i}+e^{h_i^T h_j})] + \mathbb{E}[\Delta_2],
\end{equation}
where

\begin{equation}
    \begin{split}
        \mathbb{E}[\Delta_1]&=\mathbb{E}\left[\sum_i \sum_{j\in \mathrm{pos}(i)}[h_j(\Tilde{h}_i-h_i)]\right]=0
    \end{split}
\end{equation}

and
\begin{equation}
    \begin{split}
        \mathbb{E}[\Delta_2]&=\mathbb{E}\left[ \sum_i \sum_{j\in \mathrm{neg}}\log (e^{\Tilde{h}_ih_i}+e^{\Tilde{h}_ih_j}) - \log (e^{h_i h_i}+e^{h_i h_j})\right]\\
        &\approx \mathbb{E}[ \sum_i \sum_{j\in \mathrm{neg}}[\frac{e^{h_i h_i}h_i+e^{h_i h_j}h_j}{e^{h_i h_i}+e^{h_i h_j}} (\Tilde{h}_i-h_i)\\
        &+\frac{(e^{h_i^2}h_i^2+e^{h_i h_j}h_j^2)(e^{h_i^2}+e^{h_i h_j})-(e^{h_i^2}h_i+e^{h_i h_j}h_j)^2}{2(e^{h_i h_i}+e^{h_i h_j})^2}(\Tilde{h}_i-h_i)^2]]\\
        &=\frac{1}{2}\sum_i \sum_{j\in \mathrm{neg}}\frac{(e^{h_i^2}h_i^2+e^{h_i h_j}h_j^2)(e^{h_i^2}+e^{h_i h_j})-(e^{h_i^2}h_i+e^{h_i h_j}h_j)^2}{2(e^{h_i^2}+e^{h_i h_j})^2}\mathrm{Var}(\Tilde{h}_i).
    \end{split}
\end{equation}

Thus,

\begin{equation}
    \mathbb{E}[\Tilde{\mathcal{L}}_{\mathrm{InfoNCE}}]=\mathcal{L}_{\mathrm{InfoNCE}}+\frac{1}{2}\sum_i \sum_{j\in \mathrm{neg}}\phi(h_i,h_j)\mathrm{Var}(\Tilde{h}_i),
\end{equation}
where $\phi(h_i,h_j)=\frac{(e^{h_i^2}h_i^2+e^{h_i h_j}h_j^2)(e^{h_i^2}+e^{h_i h_j})-(e^{h_i^2}h_i+e^{h_i h_j}h_j)^2}{2(e^{h_i^2}+e^{h_i h_j})^2}$.   
\end{proof}

\section{More Related Work}\label{sec:more_related}

\noindent\paragraph{Graph contrastive learning} Graph contrastive learning is an unsupervised representation learning technique that learns node representations by comparing similar and dissimilar sample pairs. Classical graph contrastive learning models, such as DGI \citep{velivckovic2018deep} and MVGRL \citep{hassani2020contrastive}, use a loss function based on mutual information estimation \citep{belghazi18mutual, hjelm2018learning} to contrast node embeddings and graph embeddings. GRACE \citep{zhu2020grace} and its variants \citep{zhu2021gca, you2020graph} strive to maximize the similarity of positive pairs while minimizing the similarity of negative pairs in augmented graphs, intending to learn more effective node embeddings. However, the computational complexity of these methods is too high for datasets, limiting their applications in large-scale graphs. To reduce computational consumption, CCA-SSG \citep{zhang2021ccassg} simplified the loss function by eliminating negative pairs. Recently, GGD \citep{zheng2022ggd} uses binary cross-entropy loss to distinguish between two groups of node samples, further reducing computational usage. Despite recent related work focusing on the scalability problem of graph contrastive learning \citep{wang2022adagcl}, these methods still need to rely on graph sampling techniques when processing graphs with millions of nodes.

\section{Implementation details}\label{sec:details}
The detailed statistics for the datasets used for transductive node classification are shown in Table \ref{tab:datasets}. We compare GCL models trained with full graph and those trained with \ourmodel on small datasets. On Arxiv and Products, most GCL models cannot perform full graph training, so we compare the performance of different scalable training methods. For all GCL models, the learned representations are evaluated by training and testing a logistic regression classifier on smaller datasets. Due to Ogbn-Arxiv, Ogbn-Products and Ogbn-Papers100M exhibits more complex characteristics, we use a three-layers MLP as classifier.

\begin{table*}[!htbp]\small
\vspace{-5mm}
	\caption{Summary of the datasets used in our experiments}.	
	\begin{tabular}{l|cccc}\toprule \centering
	\textbf{Dataset}&\textbf{Nodes}&\textbf{Edges}&\textbf{Features}&\textbf{Classes}\\
	\midrule
	Cora&2,708&5,429 &1,433&7\\
	Citeseer&3,327&4,732&3,703&6\\
	Pubmed&19,717&44,338&500&3\\
	Amazon-Photo&7,650 &238,163&745&8\\
	Amazon-Computers&13,752&491,722& 767 &10\\
        Ogbn-Arxiv&169,343 &1,157,799  & 128 &  40 \\
        Ogbn-Products&2,449,029&61,859,140&100&47\\
Papers100M&111,059,956&1,615,685,872&128&172\\
\bottomrule
	\end{tabular}
	\label{tab:datasets}
	\centering
\end{table*}

We test the performance of \ourmodel~on four GCL models: SCE\footnote{\small SCE (MIT License): \url{https://github.com/szzhang17/Sparsest-Cut-Network-Embedding}}, COLES\footnote{\small COLES (MIT License): \url{https://github.com/allenhaozhu/COLES}},GRACE\footnote{\small GRACE (Apache License 2.0):  \url{https://github.com/CRIPAC-DIG/GRACE}}, 
  CCA-SSG\footnote{\small CCA-SSG (Apache License 2.0): \url{https://github.com/hengruizhang98/CCA-SSG}}. And we compared \ourmodel~with three scalable training methods Cluster-GCN\footnote{\small Cluster-GCN (MIT License): \url{https://github.com/pyg-team/pytorch_geometric/blob/master/examples/cluster_gcn_reddit.py}}, Graphsaint\footnote{\small Graphsaint (MIT License): \url{https://github.com/pyg-team/pytorch_geometric/blob/master/examples/graph_saint.py}}, and Graphzoom\footnote{\small Graphzoom (MIT License): 
  \url{https://github.com/cornell-zhang/GraphZoom}} on large graphs.  All the algorithms and models are implemented in Python and PyTorch Geometric. Experiments are conducted on a server with an NVIDIA 3090 GPU (24 GB memory) and an Intel(R) Xeon(R) Silver 4210R CPU @ 2.40GHz.

 In Table \ref{ formula}, we present the specific formulations of both the embeddings and the loss functions that we have trained in our experiments. All models are optimized using the Adam optimizer. The hyperparameters for GCL models trained with \ourmodel~ are basically the same as those used for full graph training of GCL models. We show the main hyperparameters in Table \ref{hyperparpameter1} and \ref{hyperparpameter2}. The remaining hyperparameter settings for each GCL model are list in our code: \url{https://github.com/szzhang17/StructComp}.

\begin{table*}[!thbp]\tiny
	\caption{The compression and loss function of GCL models under \ourmodel~framework.}\label{ formula}
	\centering
\begin{tabular}{lcl}\toprule
\textbf{Model} & \textbf{Compression embedding}&\textbf{Loss function}\\ 
\midrule
\textbf{SCE}&$Z_c=X_c W$& $\mathcal{L}=\frac{\alpha}{\mathsf{Tr}(Z_c^T L^{neg}_c Z_c)}$\\[2ex]
\midrule
\textbf{COLES}&$Z_c=X_c W$&$\mathcal{L}=\mathsf{Tr}(Z_c^T L^{neg}_c Z_c) -\mathsf{Tr}(Z_c^T L_c Z_c)$ \\[2ex]
\midrule
\textbf{GRACE}& $\begin{aligned}
    \left\{
    \begin{aligned}
    Z_c & = \textbf{ReLU} ((X_c W_1)W_2), \\[1ex]
    Z'_c & = \textbf{ReLU} ((X'_c W_1)W_2),
    \end{aligned}
    \right.
    \end{aligned}$ & $\begin{aligned}
\mathcal{L}(u, v) = 
 \log \frac {e^{\phi\left(z_u, z_v \right) / \tau}} {e^{\phi\left(z_u, z_v \right) / \tau} + \sum_{k\neq u, k \in G_1} e^{\phi\left(z_u, z_k \right) / \tau} + \sum_{k \neq u, k \in G_2} e^{\phi\left(z_u,  z_k \right) / \tau}}
\end{aligned}$\\[2ex]
\midrule
\textbf{CCA-SSG}&$\begin{aligned}
    \left\{
    \begin{aligned}
    Z_c & = \textbf{ReLU} ((X_c W_1)W_2), \\
    Z'_c & = \textbf{ReLU} ((X'_c W_1)W_2),
    \end{aligned}
    \right.
    \end{aligned}$& $\mathcal{L}=\|\tilde{Z_c} -\tilde{Z'_c}\|^2_F + \lambda (\|\tilde{Z_c}^T\tilde{Z_c}-I\|^2_F+\|\tilde{Z'_c}^T\tilde{Z'_c}-I\|^2_F)$\\[2ex]
\bottomrule
\end{tabular}
\end{table*}

\begin{table*}[!hbp]\scriptsize

	\caption{Summary of the  main hyper-parameters on small datasets.}\label{tab:hyper_small}
	\centering
\begin{tabular}{lcc|cc|cc|cc|cc}\toprule
\multicolumn{1}{c}{\multirow{2}*{\textbf{Model}}}& \multicolumn{2}{c}{\textbf{Cora}} & \multicolumn{2}{c}{\textbf{Citeseer}}& \multicolumn{2}{c}{\textbf{Pubmed}}& \multicolumn{2}{c}{\textbf{Photo}}&\multicolumn{2}{c}{\textbf{Computers}}\\ 
\cmidrule(r){2-11}
  &Lr &Epoch &Lr &Epoch &Lr &Epoch &Lr &Epoch &Lr &Epoch\\ 
\midrule
SCE$_{\text{\ourmodel}}$ &0.001  &50 &0.0001&50&0.02 &25 & 0.001&20&0.001&20  \\
COLES$_{\text{\ourmodel}}$&0.001&20&0.0001&50&0.02&50&0.001 &20 &0.001&20 \\
GRACE$_{\text{\ourmodel}}$&0.001&20&0.0001&30&0.02&75&0.001&200&0.001&150\\
CCA-SSG$_{\text{\ourmodel}}$&0.001&20&0.0001&50&0.01&50&0.002&100&0.005&40\\
\bottomrule
\end{tabular}\label{hyperparpameter1}
\end{table*}

\begin{table}[!thbp]\scriptsize
\caption{Summary of the  main hyper-parameters on large datasets.}\label{hyperparpameter2}
	\centering
\begin{tabular}{lccc|ccc|ccc}\toprule
& \multicolumn{3}{c}{\textbf{Ogbn-Arxiv}} & \multicolumn{3}{c}{\textbf{Ogbn-Products}} & \multicolumn{3}{c}{\textbf{Ogbn-Papers100M}} \\ 
\cmidrule(r){2-10}
 & Lr & Epoch & Weight decay & Lr & Epoch & Weight decay & Lr & Epoch & Weight decay \\ 
\midrule
SCE$_{\text{\ourmodel}}$ & 0.0001 & 10 & 0 & 0.0001 & 5 & 0.0005 & 0.001 & 1 & 0.0005 \\
COLES$_{\text{\ourmodel}}$ & 0.0001 & 5 & 0.0005 & 0.001 & 5 & 0.0005 & 0.001 & 1 & 0.0005 \\
GRACE$_{\text{\ourmodel}}$ & 0.001 & 5 & 0.0005 & 0.001 & 5 & 0.0005 & 0.001 & 1 & 0.0005 \\
CCA-SSG$_{\text{\ourmodel}}$ & 0.0001 & 10 & 0 & 0.001 & 10 & 0 & 0.001 & 1 & 0.0005 \\
\bottomrule
\end{tabular}
\end{table}

\section{More experimental results and discussions }\label{more_exp}

\subsection{Experiments on Ogbn-Papers100M}\label{papers100m}
We conduct experiments on the Ogbn-Papers100M dataset. The experimental results are shown in Table \ref{100m}. We use StructComp to train four representative GCL models. Here, we compressed ogbn-papers100M into a feature matrix $X_c \in R^{5000*128}$ and trained GCL using StructComp. The table also presents the results of GGD trained with ClusterGCN. Although GGD is specifically designed for training large graphs, when dealing with datasets of the scale of ogbn-papers100M, it still requires graph sampling to construct subgraphs and train GGD on a large number of subgraphs. In contrast, our StructComp only requires training a simple and small-scale MLP, resulting in significantly lower resource consumption compared to GGD+ClusterGCN.

\begin{table*}[!thbp]\small
\caption{
The accuracy, training time per epoch and memory usage on the Ogbn-Papers100M dataset.}\label{100m}
\centering
    \begin{tabular}{lccc}
    \toprule
     Method    & Acc & Time & Mem \\
     \midrule
     GGD& 63.5$\pm$0.5& 1.6h&4.3GB\\
       SCE$_{\text{StructComp}}$  & 63.6$\pm$0.4&0.18s&0.1GB\\
COLES$_{\text{StructComp}}$&63.6$\pm$0.4&0.16s&0.3GB  \\
GRACE$_{\text{StructComp}}$&64.0$\pm$0.3&0.44s&0.9GB\\
      CCA-SSG$_{\text{StructComp}}$&63.5$\pm$0.2&0.18s&0.1GB\\
       \bottomrule
    \end{tabular}
\end{table*}

\subsection{Experiments on the Stability of \ourmodel}
We conduct extra experiments to study the robustness of StructComp. We randomly add 10\% of noisy edges into three datasets and perform the node classification task. The experimental results are shown in Table \ref{robustness}. On the original dataset, the models trained with StructComp showed performance improvements of 0.36, 0.40, 1.30 and 1.87, respectively, compared to the models trained with full graphs. With noisy perturbation, the models trained with StructComp showed performance improvements of 0.80, 1.27, 2.47, and 1.87, respectively, compared to full graph training. This indicates that GCL models trained with StructComp exhibit better robustness.

\begin{table*}[!thbp]\small
\caption{The results over 50 random splits on the perturbed datasets.}\label{robustness}
\centering
    \begin{tabular}{lccc}
    \toprule
     Method    & Cora & Citeseer & Pubmed  \\
     \midrule
       SCE  & 78.8$\pm$1.2&69.7$\pm$1.0&73.4$\pm$2.2\\
       SCE$_{\text{StructComp}}$  & 79.3$\pm$0.9&69.3$\pm$0.9&75.7$\pm$2.8\\
       \midrule
       COLES& 78.7$\pm$1.2&68.0$\pm$1.0&66.5$\pm$1.8\\
COLES$_{\text{StructComp}}$& 79.0$\pm$1.0&68.3$\pm$0.9&69.7$\pm$2.6
       \\
       \midrule
GRACE&77.6$\pm$1.1&64.1$\pm$1.4&64.5$\pm$1.7\\
GRACE$_{\text{StructComp}}$&78.3$\pm$0.8&69.1$\pm$0.9&66.2$\pm$2.4\\
        \midrule
        CCA-SSG&75.5$\pm$1.3&69.1$\pm$1.2&73.5$\pm$2.2\\
         CCA-SSG$_{\text{StructComp}}$&78.2$\pm$0.7&69.2$\pm$0.8&76.3$\pm$2.5\\
       \bottomrule
    \end{tabular}
\end{table*}

\subsection{Experiments on deep GNN encoder}
In order to verify the approximation quality to the diffusion matrix of StructComp, we test the performance on a deep GNN architecture called SSGC \citep{zhu2021ssgc}. We transferred the trained parameters of StructComp to the SSGC encoder for inference. For full graph training in GCL, both the training and inference stages were performed using the SSGC encoder. Table \ref{ssgc_encoder} shows our experimental results, indicating that even with a deeper and more complicated encoder, StructComp still achieved outstanding performance.

\begin{table*}[!thbp]\small
	\caption{The results of GCLs with SSGC encoders over 50 random splits.}\label{ssgc_encoder}
	\centering
\begin{tabular}{l|ccc}\toprule
\textbf{Method}& \textbf{Cora} & \textbf{Citeseer}&\textbf{Pubmed}\\ 
\midrule
SCE&    81.8$\pm$0.9&  72.0$\pm$0.9                            &78.4$\pm$2.8\\
SCE$_{\text{StructComp}}$ & 82.0$\pm$0.8 &  71.7$\pm$0.9  &  77.8$\pm$2.9\\
\midrule
COLES&   81.8$\pm$0.9    &71.3$\pm$1.1&   74.8$\pm$3.4 \\
COLES$_{\text{StructComp}}$ & 82.0$\pm$0.8   &  71.6$\pm$1.0  & 75.6$\pm$3.0\\
\midrule
GRACE&80.2$\pm$0.8  & 70.7$\pm$1.0   &77.3$\pm$2.7 \\
GRACE$_{\text{StructComp}}$&81.1$\pm$0.8  & 71.0$\pm$1.0&    78.2$\pm$1.3\\
\midrule
CCA-SSG&  82.1$\pm$0.9    &71.9$\pm$0.9 &   78.2$\pm$2.8\\
CCA-SSG$_{\text{StructComp}}$&82.6$\pm$0.7  &  71.7$\pm$0.9 &   79.4$\pm$2.6   \\
\bottomrule
\end{tabular}
\end{table*}

\subsection{Comparison with recent GCL baselines}
We provide a comparison between StructComp and recent GCL baselines \citep{blockgcl,spgcl,contrastreg,zheng2022ggd}. The specific results are shown in Table \ref{recent_work}. For SP-GCL, we are unable to get the classification accuracy on CiteSeer since it does not take isolated nodes as input. The performance and resource consumption of various GCL models trained with StructComp are superior to recent GCL baselines.

It should be noted that the goal of these studies and our work are different. The aim of SPGCL is to handle homophilic graphs and heterophilic graphs simultaneously. BlockGCL attempts to explore the application of deep GNN encoder in the GCL field. Contrast-Reg is a novel regularization method which is motivated by the analysis of expected calibration error. GGD is a GCL model specifically designed for training large graphs, it is not a training framework. On the other hand, StructComp is a framework designed to scale up the training of GCL models: it aims to efficiently train common GCL models without performance drop. It is not a new GCL model that aims to achieve  SOTA performance compared to existing GCL models. So our work is orthogonal to these previous works. In fact, StructComp can be used as the training method for SP-GCL, BlockGCL, Contrast-Reg and GGD. In future work, we will further investigate how to train these recent graph contrastive learning methods using StructComp.

\begin{table*}[thbp]\scriptsize
	\caption{The results of StructComp-trained GCLs and some GCL baselines over 50 random splits.}\label{recent_work}
	\centering
\begin{tabular}{lccc|ccc|ccc}\toprule
\multicolumn{1}{c}{\multirow{2}*{\textbf{Method}}}& \multicolumn{3}{c}{\textbf{Cora}} & \multicolumn{3}{c}{\textbf{Citeseer}}& \multicolumn{3}{c}{\textbf{Pubmed}}\\ 
\cmidrule(r){2-10}
  &Acc &Time &Mem &Acc &Time &Mem&Acc &Time &Mem \\ 
\midrule
BlockGCL&78.1$\pm$2.0&0.026&180 &64.5$\pm$2.0 &  0.023& 329 & 74.7$\pm$3.1& 0.037&986\\
SP-GCL&81.4$\pm$1.2&0.016&247& -& 0.021&319&74.8$\pm$3.2&0.041& 1420\\
Contrast-Reg&79.2$\pm$1.3& 0.048& 355 & 69.8$\pm$1.6&0.097 &602&72.4$\pm$3.5&0.334&11655\\
GGD&79.9$\pm$1.7&0.013&118&  71.3$\pm$0.7&0.018&281&74.0$\pm$2.4&0.015&311\\
SCE$_{\text{StructComp}}$&81.6$\pm$0.9&0.002&23 &71.5$\pm$1.0&0.002&59&77.2$\pm$2.9&0.003&54\\
COLES$_{\text{StructComp}}$&81.8$\pm$0.8&0.002&24&71.6$\pm$0.9&0.003&60&75.3$\pm$3.1&0.003&61\\
GRACE$_{\text{StructComp}}$&79.7$\pm$0.9&0.009&37 &70.5$\pm$1.0&0.009&72&77.2$\pm$1.4&0.009&194 \\
CCA-SSG$_{\text{StructComp}}$& 82.3$\pm$0.8&0.006&38&  71.6$\pm$0.9&0.005&71&78.3$\pm$2.5&0.006&85 \\
\bottomrule
\end{tabular}
\end{table*}

\subsection{Discussion on Graph Partitioning}

We conducted additional experiments to investigate the impact of graph partitioning on the performance of StructComp. In Table \ref{different_partition}, we demonstrate the effects of three algorithms, algebraic JC, variation neighborhoods, and affinity GS, on the performance of StructComp. These three graph coarsening algorithms are widely used in scalable GNNs \citep{huang2021scaling}, from which we can obtain the specific graph partition matrix P. The experimental results suggest that different graph partition methods has little impact on StructComp on these datasets.

\begin{table*}[!thbp]\small
\caption{The results of different graph partition methods.}\label{different_partition}
\centering
    \begin{tabular}{lccc}
    \toprule
     Method    & Cora & Citeseer & Pubmed  \\
     \midrule
       VN+SCE$_{\text{StructComp}}$ & 81.3$\pm$0.8&71.5$\pm$1.0&77.5$\pm$2.7\\
       JC+SCE$_{\text{StructComp}}$ & 81.2$\pm$0.9&71.5$\pm$1.1&77.3$\pm$2.7\\
       GS+SCE$_{\text{StructComp}}$ & 81.5$\pm$0.8&71.4$\pm$1.0&77.4$\pm$3.0\\
       METIS+SCE$_{\text{StructComp}}$ &81.6$\pm$0.9 &71.5$\pm$1.0&77.2$\pm$2.9 \\
       \midrule
       VN+COLES$_{\text{StructComp}}$ &81.4$\pm$0.9&71.6$\pm$0.9&75.5$\pm$3.0\\
       JC+COLES$_{\text{StructComp}}$ &81.4$\pm$0.9&71.5$\pm$1.0&75.3$\pm$3.0\\
       GS+COLES$_{\text{StructComp}}$ &81.8$\pm$0.8&71.6$\pm$1.0&75.5$\pm$3.2\\
       METIS+COLES$_{\text{StructComp}}$ &81.8$\pm$0.8&71.6$\pm$0.9&75.3$\pm$3.1\\
       \bottomrule
    \end{tabular}
\end{table*}

\subsection{Experiments on inductive datasets}
Our StructComp can also be used to handle inductive node classification tasks \citep{hamilton2017inductive,zeng2021decoupling}. We provide additional experiments on inductive node classification in Table \ref{tab10}. Clearly, the GCL models trained with StructComp also perform exceptionally well on inductive node classification tasks.

\begin{table*}[!thbp]\small
	\caption{The results on two inductive datasets. OOM means Out of Memory on GPU.}\label{tab10}
	\centering
\begin{tabular}{l|cccccc}\toprule
\multicolumn{1}{c}{\multirow{2}*{\textbf{Method}}}& \multicolumn{3}{c}{\textbf{Flickr}} & \multicolumn{3}{c}{\textbf{Reddit}}\\ 
  &Acc &Time &Mem &Acc &Time &Mem\\
\midrule
SCE& 50.6&0.55&8427&-&-&OOM\\
SCE$_{\text{StructComp}}$ &51.6&0.003&43&94.4&0.017&1068 \\
\midrule
COLES&50.3&0.83&9270&-&-&OOM \\
COLES$_{\text{StructComp}}$&50.7&0.003&48&94.2& 0.024 & 1175\\
\midrule
GRACE& - & - & OOM & - & - & OOM\\
GRACE$_{\text{StructComp}}$&51.5 &0.010&221 & 94.3 &0.079 & 8683\\
\midrule
CCA-SSG&51.6 & 0.125& 1672&94.9& 0.21 & 5157\\
CCA-SSG$_{\text{StructComp}}$&51.8&0.007&99&95.2&0.56&457\\
\bottomrule
\end{tabular}
\end{table*}

\subsection{Experiments on heterophilous graphs}
We conduct experiments to train SP-GCL \citep{spgcl} with StructComp, in order to verify the performance of StructComp on heterophilous graphs. The experimental results are shown in Table \ref{spgcl}. Overall, the SP-GCL trained by StructComp is superior to full graph training. This is our initial attempt to use StructComp to handle heterophilous graphs, and it is obviously a valuable direction worth further research.

\begin{table*}[!thbp]\scriptsize
	\caption{The results on heterophilous datasets.}\label{spgcl}
	\centering
\begin{tabular}{lccc|ccc|ccc}\toprule
\multicolumn{1}{c}{\multirow{2}*{\textbf{Method}}}& \multicolumn{3}{c}{\textbf{Chameleon}} & \multicolumn{3}{c}{\textbf{Squirrel}}& \multicolumn{3}{c}{\textbf{Actor}}\\
\cmidrule(r){2-10}
  &Acc &Time &Mem &Acc &Time &Mem&Acc &Time &Mem\\
\midrule
SP-GCL&65.28$\pm$0.53&0.038&739&52.10$\pm$0.67& 0.080&3623&28.94$\pm$0.69&0.041&802\\
SP-GCL$_{\text{StructComp}}$&66.65$\pm$1.63&0.011&168&53.08$\pm$1.39&0.009&217&28.70$\pm$1.25&0.013&159\\
\bottomrule
\end{tabular}
\end{table*}

\end{document}